\documentclass{article}
\usepackage{times}
\usepackage{microtype}        
\usepackage{soul}             
\usepackage{url}              
\usepackage{amsmath, amsfonts, amsthm, amssymb}  
\usepackage{graphicx}         
\usepackage{svg}
\usepackage{array}
\usepackage{cellspace}
\usepackage{multicol}
\usepackage{booktabs}         
\usepackage{makecell}         
\usepackage{multirow}
\usepackage{siunitx}          
\usepackage[table]{xcolor}    
\usepackage{arydshln}         
\usepackage{rotating}         
\usepackage{stfloats}         
\usepackage{verbatim}         
\usepackage[switch]{lineno}
\usepackage{hyperref}
\usepackage{natbib}
\newcommand{\RETURN}{\textbf{return} }
\usepackage{wrapfig}
\usepackage[preprint]{icml2026}

\newtheorem{theorem}{Theorem}

\newtheorem{corollary}{Corollary}

\theoremstyle{definition}

\newcounter{problem}

\theoremstyle{remark}
\newtheorem{remark}{Remark}

\icmltitlerunning{Neuronal Stochastic Attention Circuit (NSAC) for Probabilistic Representation Learning}

\begin{document}

\twocolumn[
  \icmltitle{\large{Neuronal Stochastic Attention Circuit (NSAC) for Probabilistic Representation Learning}}

  \icmlsetsymbol{equal}{*}

  \begin{icmlauthorlist}
    \icmlauthor{Waleed Razzaq}{yyy}
    \icmlauthor{Yun-Bo Zhao}{yyy,comp}
  \end{icmlauthorlist}

  \icmlaffiliation{yyy}{Department of Automation, University of Science \& Technology of China, Hefei, China}
  \icmlaffiliation{comp}{Institute of Artificial Intelligence, Hefei Comprehensive National Science Center}
  \icmlcorrespondingauthor{Yun-Bo Zhao}{ybzhao@ustc.edu.cn}
  \icmlcorrespondingauthor{Waleed Razzaq}{waleedrazzaq@mail.ustc.edu.cn}
  \icmlkeywords{Continuous-time Attention, Neuronal Attention Circuit (NAC), Neuronal Circuit Policies (NCPs)  }
  \vskip 0.15in
]

\printAffiliationsAndNotice{}


\begin{abstract}
\small
Reliable uncertainty quantification in continuous-time (CT) representation learning remains nascent, particularly within CT attention literature. We introduce the Neuronal Stochastic Attention Circuit (NSAC), a novel biologically-inspired CT attention architecture that reformulates attention logit computation as the solution of an Ornstein–Uhlenbeck stochastic differential equation modulated by input-dependent, nonlinear interlinked gates derived from repurposed \textit{C.\ elegans} Neuronal Circuit Policies (NCPs) wiring mechanism. It induces Gaussian distribution over logits that propagates principled stochasticity through logistic-normal distribution over attention weights to yield probabilistic output. A two-term objective function combining Gaussian negative log-likelihood with an epistemic-separation regularizer enforces higher predictive variance under distributional shifts and enables joint quantification of aleatoric and epistemic uncertainty. Theoretically, we provide: (i) state stability bounds; (ii) closed-form guarantees; and (iii) frozen-coefficient error approximation. Empirically, we implement NSAC in a diverse set of learning tasks including: (i) irregular CT function approximation; (ii) multivariate regression; (iii) long-range forecasting; (iv) Industry~4.0; and (v) lane-keeping of autonomous vehicles. We observe that the NSAC remains competitive against several baselines in terms of accuracy and produces informative uncertainty estimates while being interpretable at the neuronal cell level.
\end{abstract}

\vspace{-5mm}
\section{Introduction}
Representation learning in Continuous-time (CT) domain with reliable uncertainty quantification (UQ) is essential for high-stakes, safety-critical systems such as robotics~\cite{lechner2020neural,chahine2023robust}, medical~\cite{zhang2024trajectory} and Industry~4.0~\cite{razzaq2025carle}, where failures can lead to significant human or economic loss. Two types of uncertainty exist: (1) aleatoric uncertainty, which is induced by inherent noise in the data, and (2) epistemic uncertainty, which arises from model uncertainty. While post-hoc methods such as Gaussian maximum-likelihood estimation (GMLE)~\cite{kendall2017uncertainties}, Monte-Carlo Dropout (MCD)~\cite{gal2016dropout}, Deep-Ensembles (DE)~\cite{lakshminarayanan2017simple}, and Deep Evidential Regression (DER)~\cite{amini2020deep} exist, these frameworks lack integration with the temporal dynamics of CT architectures.\\
Scaled-Dot-Product-Attention (SDPA) mechanisms have substantially advanced sequence modeling by enabling dynamic weighting of temporal dependencies \cite{vaswani2017attention}. However, the inherently discrete nature of attention computation is not well suited for CT domain. Neuronal Attention Circuit (NAC)~\cite{razzaq2025neuronal} address this limitation by rethinking attention logit computation as the solution to a linear ODE modulated by input-dependent, nonlinear interlinked gates derived from repurposing the wiring mechanism of \textit{C.elegans} nematode's Neuronal Circuit Policies (NCPs)~\cite{lechner2018neuronal}. While NAC provides solver-free scalability, its deterministic formulation prevents it from providing uncertainty estimates.\\
In this work, we extend the foundational principles of NAC by reframing attention logit computation as the solution to an Ornstein–Uhlenbeck stochastic differential equation (OU-SDE)~\cite{maller2009ornstein}, where the mean~($\phi$), mean-reversion rate ($\kappa$), and diffusion term~($\psi$) are modulated by input-dependent, nonlinear, interlinked gates derived from repurposed NCPs. We refer to this formulation as the Neuronal Stochastic Attention Circuit (NSAC). NSAC yields a Gaussian distribution over temporal attention logits and propagates stochasticity through a logistic-normal distribution over attention weights to compute attention scores and probabilistic outputs. NSAC learns quantification of both aleatoric and epistemic uncertainty via a two-term training objective combining a Gaussian negative log-likelihood (NLL) with an epistemic separation regularizer. We provide rigorous theoretical guarantees including: (i) mean-process state stability; (ii) closed-form finite-time bounds; and (iii) frozen-coefficient error approximations. We implement the NSAC in a diverse set of learning tasks, including: (i) irregular CT function approximation; (ii) multivariate regression; (iii) multivariate long-range forecasting; (iv) industrial prognostics; and (v) the lane-keeping of autonomous vehicles. NSAC remains competitive in terms of accuracy and produces informative uncertainty tubes against several state-of-the-art UQ baselines. To summarize, our contributions are as follows: \vspace{-2mm}
\begin{enumerate}
    \item To the best of our knowledge, we are the first to directly incorporate stochasticity into CT attention mechanism, enabling principled probabilistic CT representation learning.\vspace{-2mm}
    \item We derive a two-term objective function that regularizes predictive variance and supports uncertainty decomposition into aleatoric and epistemic components.\vspace{-2mm}
    \item We establish theoretical bounds including: (i) SDE mean-process state stability; (ii) closed-form time guarantees; and (iii) frozen-coefficient error approximations. \vspace{-2mm}
    \item We validate the proposed method across diverse representation learning tasks: (i) irregular CT function approximation; (ii) multivariate regression; (iii) multivariate long-range forecasting; (iv) industry~4.0 prognostics; and (v) lane-keeping of autonomous vehicles.\vspace{-3mm}
\end{enumerate}
The rest of the paper is organized as follows. Section~\ref{sec:literature} reviews the related literature and highlights the positioning of the proposed method. Section~\ref{sec:NSAC} provides comprehensive methodological details. Section~\ref{sec:evaluation} presents a comprehensive evaluation across a range of representation tasks. Section~\ref{sec:discuss} discusses the results and limitations. Finally, Section~\ref{sec:conclude} concludes the paper.

\section{Related Works}\label{sec:literature}
\textbf{Post-hoc Methods:} A large class of UQ methods augments deterministic predictors with a post-hoc mechanism for predictive variance after feature extraction. Gaussian maximum-likelihood estimation (GMLE)~\cite{kendall2017uncertainties} introduces heteroscedastic output variance but captures only aleatoric uncertainty. Bayesian approximation, such as Monte-Carlo Dropout (MCD)~\cite{gal2016dropout}, estimates uncertainty through stochastic inference-time masking, coupling uncertainty estimation to a regularization mechanism rather than explicitly modeling predictive stochasticity. Deep Ensembles (DE)~\cite{lakshminarayanan2017simple} improve calibration through aggregation over independently trained predictors but require linearly increasing memory and computational cost. Deep Evidential Regression (DER)~\cite{amini2020deep} instead parameterizes predictive evidence through a Normal-Inverse-Gamma (NIG) prior, enabling closed-form uncertainty estimates; however, its evidential output is still derived from deterministic latent representations. Across these approaches, uncertainty is imposed after representation learning, leaving the internal attention or feature-composition dynamics unchanged.\\
\textbf{CT Attention:} CT attention mechanisms address the temporal adaptivity of the SDPA by replacing the discrete layerwise transformation with a dynamical systems formulation. Prior work has modeled joint hidden-state and attention evolution through neural ODEs~\cite{chien2021continuous}, incorporated latent continuous dynamics into time-aware attention mechanisms~\cite{chen2023contiformer}, and parameterized attention matrices directly through partial differential equations (PDEs)~\cite{zhang2025continuous}. NAC~\cite{razzaq2025neuronal} interprets attention logits as the solution of a linear ODE modulated by an input-dependent nonlinear gating function inspired by NCPs~\cite{lechner2018neuronal}. These formulations improve temporal expressivity but remain fundamentally deterministic, providing no explicit mechanism for calibrated uncertainty propagation through attention itself.\\
\textbf{Neural SDE:} Neural SDE models incorporate stochasticity directly into latent dynamics. SDE-Net~\cite{kong2020sde} decomposes hidden evolution drift and diffusion components to capture predictive uncertainty, whereas Latent SDE~\cite{li2020scalable} extends this framework to variational CT latent trajectory modeling. These methods provide stochastic representations but require numerical SDE solvers such as \textit{Euler-Maruyama} and adjoint-based optimization, introducing substantial computational overhead and discretization error. More importantly, stochasticity is confined to the hidden state rather than attention generation, leaving uncertainty disconnected from the mechanism governing token interaction.\\
\textbf{Positioning of our work:} Our work addresses these gaps by unifying CT attention dynamics and uncertainty modeling within a single formulation. Specifically, NSAC extends NAC by modeling attention-logit evolution as an OU-SDE whose closed-form margin admits Gaussian distribution without requiring numerical simulation. This makes uncertainty an intrinsic architectural property of attention generation rather than a post-hoc method, enabling joint modeling of aleatoric and epistemic through two-term objective function while preserving solver-free computational efficiency.

\section{Methods}\label{sec:NSAC} 

\subsection{Preliminaries}
\subsubsection{Neuronal Attention Circuit (NAC)}
Neuronal Attention Circuit (NAC) is a biologically inspired CT attention mechanism that reformulates attention logit computation as the solution of an input-dependent first-order ODE:
\begin{equation}
\frac{da_t}{dt} = -{\omega_\tau(\mathbf{u})}a_t + {\gamma(\mathbf{u})}, 
\end{equation}
where $\mathbf{u}=[\mathbf{q};\mathbf{k}]$ denotes the sparse Top-\emph{K} query-key concatenated representation obtained through a square-root block-partitioning strategy~\cite{child2019sparse}, reducing attention complexity to $O(n\sqrt{n})$. $\omega_{\tau}$ is a \textit{learnable time-constant gate} and $\gamma$ is a \textit{content-target gate}. Both gates are derived using repurposed \textit{C. elegans} Neuronal Circuit Policies (NCPs)~\cite{lechner2018neuronal}. NCPs consist of a sparse fixed connectome consisting of sensory, interneuron, command, and motor neurons. NAC repurposes this into a flexible modifiable architecture that splits sensory neurons ($\mathcal{NN}_{\text{sensory}}$) for query~($q$), key~($k$), and value~($v$) projections, and the inter-to-motor pathways form a shared backbone ($\mathcal{NN}_{\text{backbone}}$) for computing $\gamma$ and $\omega_{\tau}$. Under locally frozen-coefficient approximation\cite{john1952integration}, NAC admits a closed-form solution:
\begin{equation}
a_t=\frac{\gamma}{\omega_{\tau}}
+\left(a_0-\frac{\gamma}{\omega_{\tau}}\right)e^{-\omega_{\tau}t},
\end{equation}
enabling an efficient solver-free forward pass~\cite{razzaq2025neuronal}.

\subsection{Neuronal Stochastic Attention Circuit (NSAC)}
We propose to view the computation of attention logits as a stochastic Gaussian distribution, interpreting them as the solution to an OU-SDE, modulated by input-dependent, nonlinear, interlinked gates derived from repurposed \textit{C.elegans} nematode NCP's wiring mechanism:
\begin{equation}
\begin{split}
da_t =
\overbrace{\underbrace{f_\kappa([\mathbf{q};\mathbf{k}],t,\theta_\kappa)}_{\kappa(\mathbf{u})} \cdot
(\underbrace{f_\phi([\mathbf{q};\mathbf{k}],t,\theta_\phi)}_{\phi(\mathbf{u})} - a_t)~dt}^{\text{drift}} ~+ \\ \overbrace{\underbrace{f_\psi([\mathbf{q};\mathbf{k}],t,\theta_{\psi})}_{\psi(\mathbf{u})}
\, d\mathcal{W}_t}^{\text{diffusion}}
\end{split}
\label{eq:lq_diff}
\end{equation}
Here, $\mathbf{u} = [\mathbf{q}; \mathbf{k}]$ is the sparse Top-\emph{K} concatenated input, $\kappa(\mathbf{u})$ denotes the learnable mean-reversion rate gate with parameters $\theta_\kappa$, $\phi(\mathbf{u})$ represents the learnable long-term mean attention gate;, and $\psi(\mathbf{u})$ gate controls the Brownian fluctuations ($\mathcal{W}_t$) in the logit trajectory. We refer to this formulation as the Neuronal Stochastic Attention Circuit (NSAC). It enables the logits to mirror input-dependent temporal dynamics found in \textit{C.elegans} nematode while simultaneously producing probabilistic outputs.\\
\textbf{Motivation behind this formulation:} The choice of the OU formulation is motivated by its established role in modeling mean-reverting dynamics in biological systems, particularly neuronal membrane potential fluctuations~\cite{ricciardi1979ornstein, jahn2011motoneuron, ditlevsen2005estimation} and synaptic noise~\cite{fellous2003synaptic}. The neural circuits of \textit{C.~elegans} exhibit similar equilibrium-seeking dynamics~\cite{iwasaki2004stochastic}. This behavior also appears to extend across scales, as attention-modulated circuits in primates show comparable mean-reverting dynamics~\cite{feng2012extensions}. Furthermore, a key advantage of the OU formulation is its closed-form solution, which avoids the discretization errors associated with numerical SDE solvers such as the \textit{Euler--Maruyama} method.\\
\textbf{Forward- pass of the NSAC using Closed-form Solution:} To obtain the attention logit state ($a_t$), the NSAC exploits the forward pass through its closed-form solution. Although the gates are time-varying and input-dependent, they evolve on discrete internal updates. Consequently, from the perspective of CT dynamics, the gates remain constant within each interval and change only at step boundaries. This separation of time scales induces a piecewise-constant (locally frozen) parameterization, reducing the dynamics to a system with constant coefficients over each interval. As a result, the forward pass admits an efficient closed-form update while retaining temporal variability across steps. Given an initial condition $a_0$, the closed-form solution for the NSAC can be written as follows:
\begin{equation}
a_t = a_0 e^{-\kappa t} + \phi (1 - e^{-\kappa t}) + \psi \int_0^t e^{-\kappa (t-s)} \, d\mathcal{W}_s
\end{equation}
The expected value and variance can be expressed as:
\begin{equation}
\begin{aligned}
\mu_t &= \mathbb{E}[a_t] = \phi + (a_0 - \phi)e^{-\kappa t} \\
\sigma_t^2 &= \mathrm{Var}[a_t] = \frac{\psi^2}{2\kappa}\left(1 - e^{-2\kappa t}\right)
\label{eq:ou_mean_variance}
\end{aligned}
\end{equation}
The derivation is available in Appendix~\ref{appendix:closed-form}. Further analysis on frozen-coefficient is available in Appendix~\ref{appendix:frozen_coefficient}.\vspace{-3mm}

\subsection{Stability Analysis}
We now investigate the stability bounds under both SDE-based and closed-form formulations.
\begin{theorem}[SDE Mean-Process State Stability] \label{theorem:state_stability}
The logits are stochastic in nature, so we analyze the stability of the mean process \(\mu_t^{(i)} = \mathbb{E}[a_t^{(i)}]\). Let $\mu_t^{(i)} = \mathbb{E}[a_t^{(i)}]$ denote the mean of the $i$-th
attention logit, governed by $d\mu_t^{(i)}/{dt} = \sum_{j=1}^{M} \kappa_{i,j}\left(\Phi_{i,j} - \mu_t^{(i)}\right), \ \text{where }\kappa_{i,j} > 0$ and $M$ are incoming connections. Define $\Phi_{\min} =\min_{j}\,\Phi_{i,j}, \Phi_{\max} = \max_{j}\,\Phi_{i,j}$. Then, for any finite horizon $t \in [0,T]$ and any initial condition
$\mu_0^{(i)} \in [\Phi_{\min},\,\Phi_{\max}]$,
\begin{equation}
  \Phi_{\min} \;\leq\; \mu_t^{(i)} \;\leq\; \Phi_{\max}.
  \label{eq:stability-bound}
\end{equation}
In the special case $M=1$, the dynamics reduce to
$d\mu_t^{(i)}/dt = \kappa(\phi - \mu_t^{(i)})$ with unique equilibrium $\phi = \Phi_{i,1}$. Bound~\ref{eq:stability-bound} holds for any
$\mu_0^{(i)} \in [\phi_{\min},\,\phi_{\max}]$, where
$\phi_{\min} = \inf f_{\phi}([\mathbf{q_i};\,\mathbf{k_1}])$ and
$\phi_{\max} = \sup f_{\phi}([\mathbf{q_i};\,\mathbf{k_1}])$
over all admissible inputs. $f_{\phi}$ applies a $\tanh$ nonlinearity, $[\phi_{\min},\,\phi_{\max}] \subseteq (-1,1)$, and the
implementation's initialization $\mu_0^{(i)} = 0$ is admissible.\\
The proof is available in Appendix~\ref{appendix:state_stability}
\end{theorem}

\subsubsection{Closed-Form error analysis}
We now examine the asymptotic stability, distributional error characterization, and exponential boundedness of the closed-form formulation.\\
\textit{Mean Error: }Define the instantaneous mean error $\epsilon_{\mu}(t) = \mu_t - \phi$, that measures the distance of the mean from its long-term equilibrium. From Eqn.~\ref{eq:ou_mean_variance}, with initial condition $a_0$, we have $\epsilon_{\mu}(t) = (a_0 - \phi)e^{-\kappa t}$. The point-wise absolute mean error is therefore
\begin{equation}
\left|\epsilon_{\mu}(t)\right| = |a_0 - \phi|\,e^{-\kappa t},
\end{equation}
showing that convergence of the mean follows an exact exponential law controlled by the rate parameter $\kappa$. This yields the following finite-time guarantees:
\begin{corollary}[Exponential Mean Decay Bound]\label{corollary:exponential}
If $\kappa > 0$, then for all $t \ge 0$,
\begin{equation}
|\mu_t - \phi| \;\leq\; |a_0 - \phi|\,e^{-\kappa t}
\end{equation}
\end{corollary}
\begin{remark}
If $\kappa > 0$, then $\lim_{t \to \infty} e^{-\kappa t} = 0,$ and therefore $\lim_{t \to \infty} \ \text{,} \ \ \mu_t = \phi$. Hence, the mean converges exponentially to its equilibrium value $\phi$ at rate $\kappa$. If $\kappa < 0$, then $e^{-\kappa t} = e^{|\kappa|t}$, which diverges as $t \to \infty$. Consequently, the mean grows exponentially away from $\phi$ in magnitude. If $\kappa = 0$, then $d\mu_t/{dt} = 0$, so that $\mu_t$ remains constant for all $t$. Therefore, for bounded dynamics that converge to an interpretable long-term mean, it is necessary that $\kappa > 0$.
\end{remark}
\begin{corollary}[Variance Boundedness and Monotone Growth]
If $\kappa > 0$ and $\psi > 0$, then the $\sigma_t^2$ is monotonically increasing in $t$ and satisfies
\begin{equation}
0 \;\leq\; \sigma_t^2 \;\leq\; \sigma_\infty^2 \;=\; \frac{\psi^2}{2\kappa}, 
\end{equation}
with $\lim_{t \to \infty} \sigma_t^2 = \sigma_\infty^2$.
\end{corollary}
\begin{remark}
The stationary variance $\sigma_\infty^2 = \frac{\psi^2}{2\kappa}$ is determined jointly by the diffusion scale $\psi$ and the mean-reversion rate $\kappa$. A larger value of $\kappa$ reduces the stationary variance, whereas a larger value of $\psi$ amplifies it. Consequently, $\kappa$ plays a dual role: (i) It governs the speed of mean reversion (Corollary~\ref{corollary:exponential}); (ii) It determines the ceiling of the stationary variance. Thus, increasing $\kappa$ simultaneously accelerates convergence toward the equilibrium mean and suppresses long-run stochastic fluctuations.
\end{remark}
\begin{corollary}[Uniform Initialization]
If the initialization is known only to belong to a bounded set, i.e., $|a_0 - \phi| \le \mathcal{M} \text{ for some } \mathcal{M} > 0 $,
then the mean error admits the uniform bound
\begin{equation}
|\mu_t - \phi| \le \mathcal{M} e^{-\kappa t}, \qquad \forall t \ge 0.
\end{equation}
\end{corollary}
\begin{remark}
This highlights that exponential convergence of the mean holds uniformly across all admissible initial conditions, with the constant $\mathcal{M}$ capturing the worst-case initial deviation.
\end{remark}
\begin{corollary}[Convergence Time to Accuracy]
For a target tolerance $\delta > 0$, solving $|a_0 - \phi| e^{-\kappa t} \le \delta$ yields the convergence time threshold
\begin{equation}
t \;\geq\; \frac{1}{\kappa}\ln\frac{|a_0 - \phi|}{\delta}
\end{equation}
\end{corollary}
\begin{remark}
The convergence time of the mean is inversely proportional to $\kappa$, and scales only logarithmically with $\delta$. Intuitively, a larger $\kappa$ accelerates contraction of the mean toward equilibrium, leading to faster attainment of any desired tolerance.
\end{remark}

\subsection{Designing NSAC as Neural Network Layer}
We now provide the design steps to convert the NSAC into a neural network layer informed by preceding stability analysis. It consists of 6 steps: (i) gating structures; (ii) handling time ($t$); (iii) stochastic attention weights; (iv) stochastic attention scores; (v) scaling to multi-head; (vi) probabilistic output. An illustration of the internal architecture of NSAC layer is provided in Figure~\ref{fig:nsac_architecture}.
\begin{wrapfigure}{r}{0.45\columnwidth}
\centering
\includegraphics[width=0.22\textwidth]{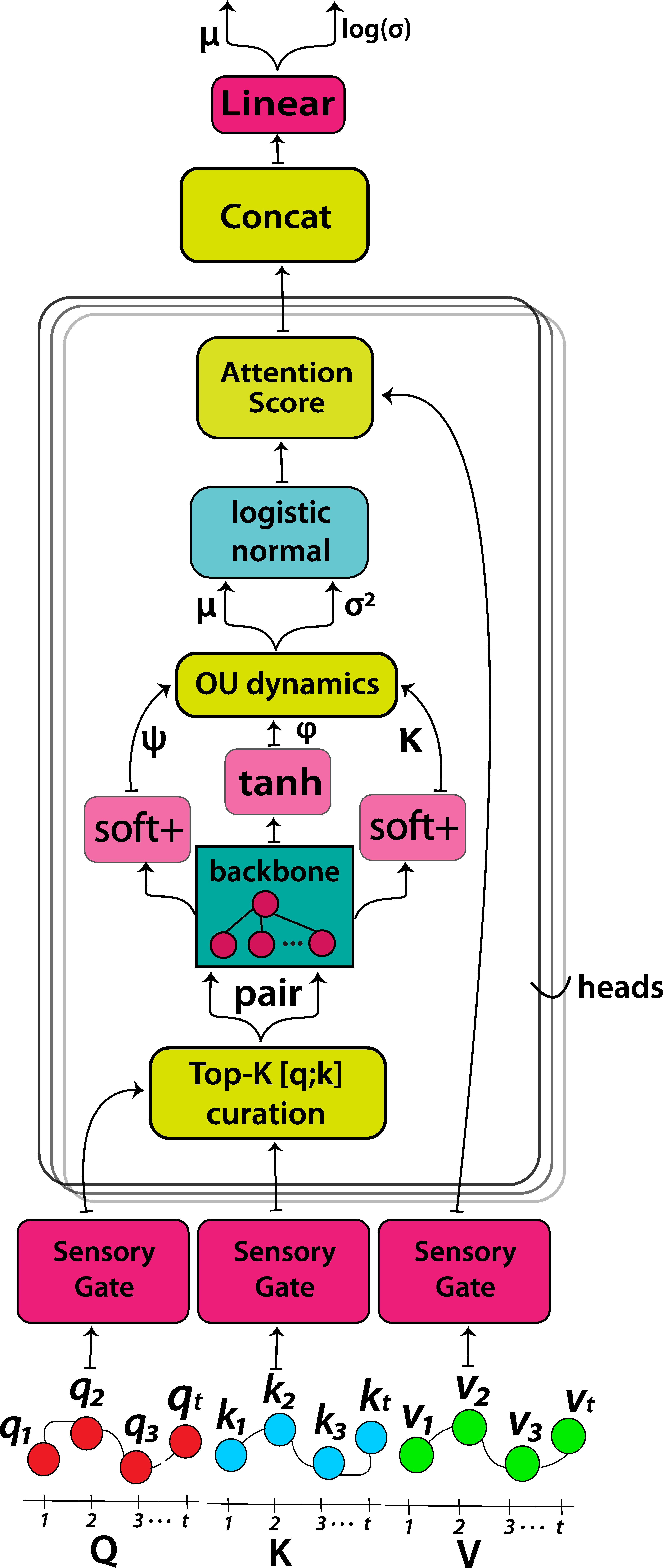}
\caption{Internal architecture of NSAC.}
\label{fig:nsac_architecture}
\vspace{-8pt}
\end{wrapfigure}
\textbf{Gating Structures:} $\mathcal{NN}_{\text{sensory}}$ is used to project the $q$, $k$, and $v$ vectors. This design produces structured, context-aware representations rather than collapsing inputs through a linear layer, thereby preserving locality and modularity and enhancing information routing. \textit{Intuitively}, these functions like biological receptors, converting external stimuli into organized signals that feed downstream circuits.\\
To modulate stochastic dynamics, we modified $\mathcal{NN}_{\text{backbone}}$ to produce three output heads producing $\kappa$, $\phi$, and $\psi$ and integrate neuronal activity over time to generate outputs that regulate signal flow. To satisfy the stability requirements, we enforce positivity on $\kappa$ and $\psi$ using a $\mathrm{softplus}$ and apply $\mathrm{tanh}$ nonlinearity to $\phi$ to allow passing negative logits. Rather than learning these parameters independently, they share a common representation, which offers several advantages: (i) faster convergence during training and (ii) improved modeling of temporal and structural dependencies. The stochastic attention logit mean ($\mu_{ou}$) and variance ($\sigma^2_{ou}$) are updated using Eqn.~\ref{eq:ou_mean_variance} under initial condition $a_0 = 0$:
\begin{equation}
\begin{aligned}
\mu_{ou} &= \phi (1 - e^{-\kappa t}) \\
\sigma^2_{ou} &= \frac{\psi^2}{2\kappa}\left(1 - e^{-2\kappa t}\right)
\end{aligned}
\label{eq:ou}
\end{equation}
\textbf{Handling Time ($t$):} The NSAC adapts its temporal dynamics to the task following \cite{hasani2022closed} by constructing a per-head internal normalized time vector
$t = \sigma(-t_a t_{sample} + t_b)$, where $t_a$ and $t_b$ are learnable affine parameters and $t_{sample}$ is derived from the input sample. When the task contains meaningful temporal structure, $t_{sample}$ introduces input-dependent modulation of the dynamics. When temporal information is not meaningful, $t_{sample}$ is set to 1, and $t$ reduces to a learned vector through $\sigma(-t_a + t_b)$, allowing the model to adaptively regulate the evolution rate of its internal stochastic dynamics. In both cases, the $\sigma$ ensures $t \in [0,1]$, supporting stable stochastic dynamics modeling.\\
\textbf{Stochastic Weights:} The stochastic attention logits are modeled as Gaussian random variables per sample, with independent Brownian motion realizations per logit: 
$a_t = \mu_{\text{ou}} + \sigma_{\text{ou}} \epsilon$, where $\epsilon \sim \mathcal{N}(0,1)$. The logits are then normalized using $\alpha_t = \mathrm{softmax}(a_t)$, resulting in a logistic-normal distribution over the attention weights~\cite{atchison1980logistic}. This formulation preserves OU-driven stochastic dynamics while remaining fully differentiable.\\
\textbf{Stochastic Attention Scores:} The attention scores are then computed by multiplying the \(\alpha_t\) with the value vector \(v_t\):
\begin{equation}
\text{NSAC}(q, k, v) = \alpha_t \cdot v_t
\end{equation}
\textbf{Scalling to Multihead:} To scale this mechanism to multihead attention, we divide the input sequence into $H$ independent subspaces (heads) of dimension $d_\text{model}/H$, yielding query, key, and value tensors $(q^{(h)}, k^{(h)}, v^{(h)})$ for $h \in \{1, \dots, H\}$. For each head, the pairwise mean ($\mu_{ou}^{(h)}$) and variance ($\sigma_{ou}^{{(h)}^2}$) are computed using Eqn.~\ref{eq:ou}, followed by the logistic-normal distribution to calculate the stochastic attention weights $\alpha^{(h)}_t$, which are then multiplied with the value vector $v^{(h)}$, producing head-specific attention scores. Finally, these scores are concatenated and linearly projected into twice the $d_\text{model}$ dimension.\\
\textbf{Probabilistic Output:} A linear projection splits the learned representation into two $d_\text{model}$ dimension streams, producing a predictive mean ($\mu_{out}$) and log-standard- deviation ($s_{out}$). These define a heteroscedastic Gaussian predictive distribution, allowing the model to produce probabilistic outputs.
\begin{equation}
p(y \mid x) = \mathcal{N}\big(\mu_{out},\sigma^2_{out}\big), 
\quad 
\sigma_{out} = e^{(s_{out})}. 
\end{equation}
\vspace{-6mm}

\begin{algorithm}[htbp]
\small
\caption{Training Algorithm for NSAC}
\label{algo:train_step_nsac_expectation}
\begin{algorithmic}
\REQUIRE Data $(x_{\mathrm{ID}}, y_{\mathrm{ID}})$, NSAC circuit $f_\theta$, optimizer $\mathcal{O}$, MC samples $N_{\text{mc}}$, OOD param. $(\mu_{\mathrm{\mathrm{pert}}}, \sigma_{pert})$
\ENSURE Total loss: $\mathcal{L}_{\text{total}}$
\STATE Surrogate OOD samples: $x_{\mathrm{OOD}} \gets x_{\mathrm{ID}} + \mathcal{N}(\mu_{pert}, \sigma_{pert})$
\FOR{$n = 1$ to $N_{\text{mc}}$}
    \STATE ID: $(\mu_{\mathrm{out}\sim \mathrm{ID}}^{(n)}, \sigma_{\mathrm{out}\sim \mathrm{ID}}^{(n)}) \gets f_\theta(x_{\mathrm{ID}})$
    \STATE OOD: $(\mu_{\mathrm{out}\sim \mathrm{OOD}}^{(n)}, \sigma_{\mathrm{out}\sim \mathrm{OOD}}^{(n)}) \gets f_\theta(x_{\mathrm{OOD}})$
\ENDFOR
\STATE Compute losses:  
\STATE \quad $\mathcal{L}_{\text{nll}} \gets \mathbb{E}[\text{NLL}(y_{\mathrm{out}\sim \mathrm{ID}}, \mu_{\mathrm{out}\sim \mathrm{ID}}, \sigma_{\mathrm{out}\sim \mathrm{ID}})]$  
\STATE \quad $\mathcal{L}_{\text{reg}} \gets \mathbb{E}[\text{Reg}(\text{Var}[\mu_{\mathrm{out}\sim \mathrm{ID}}], \text{Var}[\mu_{\mathrm{out}\sim\mathrm{OOD}}]]$
\STATE Total loss: $\mathcal{L}_{\text{total}} \gets \mathcal{L}_{\text{nll}} + \lambda  \mathcal{L}_{\text{reg}}$
\STATE Update: $\theta \gets \theta - \nabla_\theta \mathcal{L}_{\text{total}}$ using $\mathcal{O}$
\STATE \RETURN $\mathcal{L}_{\text{total}}$
\end{algorithmic}
\end{algorithm}

\subsection{Learning Uncertainty through Stochasticity}
NSAC learns predictive uncertainty through a two-term objective combining predictive accuracy and epistemic uncertainty separation. Predictive accuracy is enforced via the Gaussian negative log-likelihood (\(\mathcal{L}_{\mathrm{NLL}}\)) evaluated on in-distribution (ID) data samples (\(\mathcal{D}_{\mathrm{ID}}\)):
\begin{equation}
\mathcal{L}_{\mathrm{nll}} =
\frac{1}{2} \, \mathbb{E}_{(x,y)\sim\mathcal{D}_{\mathrm{ID}}} 
\Big[ \log(2\pi) + 2s_{out} + \frac{(y - \mu_{out})^2}{e^{2s_{out}}} \Big]
\end{equation}
To encourage separation between ID and out-of-distribution (OOD), surrogate OOD samples are generated per head using Gaussian perturbations:
\begin{equation}
x_{\mathrm{OOD}} = x_{\mathrm{ID}} + \xi, 
\quad 
\xi \sim \mathcal{N}(\mu_{\mathrm{pert}}, \sigma_{\mathrm{pert}}^2)
\end{equation}
This serves as a scalable, domain-agnostic mechanism for inducing discriminative variance separation. For each input, \(N\) stochastic forward passes through stochastic weights produce predictive means \(\{\mu^{(n)}_{out}\}_{n=1}^{N}\). The variance across these stochastic predictions is used as a measure of model-induced epistemic uncertainty. Compared with ID inputs, the regularization term promotes greater uncertainty for perturbed inputs:
\begin{equation}
\mathcal{L}_{\mathrm{reg}}
=
\log
\left(
1 +
\frac{
\mathbb{E}_{x\sim\mathcal{D}_{\mathrm{ID}}}\big[\mathrm{Var}[\mu^{(n)}_{out\sim\mathrm{ID}}]\big]
}{
\mathbb{E}_{x\sim\mathcal{D}_{\mathrm{OOD}}}\big[\mathrm{Var}[\mu^{(n)}_{out\sim\mathrm{OOD}}]\big] + \epsilon
}
\right)
\end{equation}
where \(\epsilon\) ensures numerical stability. The final training objective is:
\begin{equation}
\mathcal{L} = \mathcal{L}_{\mathrm{nll}} + \lambda \, \mathcal{L}_{\mathrm{reg}}
\end{equation}
with \(\lambda\) controlling the strength of the regularizer. This objective encourages accurate predictions on \(\mathcal{D}_{\mathrm{ID}}\) data while increasing epistemic uncertainty under \(\mathcal{D}_{\mathrm{OOD}}\), improving uncertainty separation and robustness.

\section{Evaluation}\label{sec:evaluation} 
We evaluate NSAC across diverse set of learning domains including: (i) irregular CT function approximation; (ii) multivariate regression; (iii) multivariate long-range forecasting (iv) industry~4.0 prognostics; and (v) lane-keeping of autonomous-vehicles. Additional experiments including: (i) detailed ablation analysis on key parameters; (ii) closed-loop analysis of AVs; and (iii) scalability analysis are provided in Appendix~\ref{appendix:additional_experiments}.\\
\textbf{Baselines:} We compare NSAC against several state-of-the-art UQ techniques: (i) GMLE~\cite{kendall2017uncertainties}; (ii) MCD~\cite{gal2016dropout}; (iii) DE~\cite{lakshminarayanan2017simple}; (iv) DER~\cite{amini2020deep}; and (v) SDE-Net~\cite{kong2020sde}. For a fair comparison, the penultimate layer of all post-hoc baselines uses NAC with the same model dimensions ($d_{\text{model}}$), attention head ($num_{heads}$), sparsity, and Top-\emph{K} as NSAC.\\
\textbf{Test Setup and Metrics:} We train each model with five random seeds and report the mean and standard deviation of the Mean-Squared-Error (MSE) for predictive accuracy, Negative Log-likelihood (NLL) for both prediction error and uncertainty estimation, Continuous Ranked Probability Score (CRPS) for probabilistic forecast quality, and the Expected Calibration Error (ECE) for calibration of uncertainty estimates. CRPS measures the accuracy of the predicted distribution, while ECE captures how well predicted probabilities are calibrated. Ideally, both should be low; in practice, a balance is necessary. Sharper distributions can improve CRPS but may worsen ECE, whereas broader distributions tend to improve calibration at the expense of precision. More detailed definitions of metrics are provided in Appendix~\ref{appendix:metrics} and data descriptions, preprocessing, and neural network details are available in Appendix~\ref{appendix:experimental_detail}. Table~\ref{tab:results} reports all the quantitative results. 

\begin{table*}[t]
\centering
\caption{Quantitative analysis of all methods.}
\resizebox{0.97\textwidth}{!}{%
\begin{tabular}{l|ccccccccccc}
\toprule
\multirow{2}{*}{\textbf{Model}} & \multirow{2}{*}{\textbf{Metrics}} & \multirow{2}{*}{\textbf{Spiral}} & \multicolumn{2}{c}{\textbf{Multivariate Regression}} & \multicolumn{2}{c}{\textbf{Long-range Forecasting}} & \multicolumn{3}{c}{\textbf{Industry~4.0}}& \multicolumn{2}{c}{\textbf{Autonomous Vehicles}}\\

\cmidrule(lr){4-5}  \cmidrule(lr){6-7} \cmidrule(lr){8-10} \cmidrule(lr){11-12}
 & & & \textbf{Boston} & \textbf{Kin8nm} & \textbf{ETTm1} & \textbf{J.Climate} & \textbf{XJTU-SY} & \textbf{PRONOSTIA} & \textbf{HUST} & \textbf{Udacity} & \textbf{CarRacing} \\
\midrule

\multirow{4}{*}{\rotatebox{45}{\small \textbf{NSAC}}} 
& MSE~$(\downarrow)$ 
& \textbf{0.0002\textsuperscript{\scriptsize $\pm$0.0001}}
& \textbf{0.0301\textsuperscript{\scriptsize $\pm$0.0061}}
& \textbf{0.0327\textsuperscript{\scriptsize $\pm$0.0004}}
& 0.0099\textsuperscript{\scriptsize $\pm$0.0042}
& \textbf{0.1675\textsuperscript{\scriptsize $\pm$0.0627}}
& \textbf{0.0048\textsuperscript{\scriptsize $\pm$0.0018}}
& 0.0294\textsuperscript{\scriptsize $\pm$0.0123}
& \textbf{0.0033\textsuperscript{\scriptsize $\pm$0.0008}}
& \textbf{0.0249\textsuperscript{\scriptsize $\pm$0.0068}}
& 0.0154\textsuperscript{\scriptsize $\pm$0.0064}
\\

& NLL~$(\downarrow)$ 
&-2.3483\textsuperscript{\scriptsize $\pm$0.0536}
&\textbf{-0.2665\textsuperscript{\scriptsize $\pm$0.1597}}
&\textbf{-0.3870\textsuperscript{\scriptsize $\pm$0.0596}}
&\textbf{-1.3816\textsuperscript{\scriptsize $\pm$0.3152}}
&\textbf{-1.9114\textsuperscript{\scriptsize $\pm$0.2946}}
&-0.0219\textsuperscript{\scriptsize $\pm$0.0012}
&-1.7521\textsuperscript{\scriptsize $\pm$0.1832}
&-1.5889\textsuperscript{\scriptsize $\pm$0.1378}
&-0.4122\textsuperscript{\scriptsize $\pm$0.1679}
&\textbf{-1.6450\textsuperscript{\scriptsize $\pm$0.2422}} \\

& CRPS~$(\downarrow)$ 
& \textbf{0.0095\textsuperscript{\scriptsize $\pm$0.0021}}
& 0.0940\textsuperscript{\scriptsize $\pm$0.0093}
& \textbf{0.1029\textsuperscript{\scriptsize $\pm$0.0007}}
& 0.1037\textsuperscript{\scriptsize $\pm$0.0202}
& \textbf{0.1422\textsuperscript{\scriptsize $\pm$0.0977}}
& \textbf{0.0261\textsuperscript{\scriptsize $\pm$0.0049}}
& \textbf{0.0310\textsuperscript{\scriptsize $\pm$0.0032}}
& 0.0231\textsuperscript{\scriptsize $\pm$0.0031}
& \textbf{0.0814\textsuperscript{\scriptsize $\pm$0.0075}	}
& 0.7271\textsuperscript{\scriptsize $\pm$0.2619}\\

& ECE~$(\downarrow)$ 
& 0.1631\textsuperscript{\scriptsize $\pm$0.0435}
& 0.3974\textsuperscript{\scriptsize $\pm$0.0334} 
& 0.3085\textsuperscript{\scriptsize $\pm$0.0084}
& 0.3970\textsuperscript{\scriptsize $\pm$0.1024}
& 0.3179\textsuperscript{\scriptsize $\pm$0.0609}
& \textbf{0.1404\textsuperscript{\scriptsize $\pm$0.0098}}
& \textbf{0.1337\textsuperscript{\scriptsize $\pm$0.0106}}
& \textbf{0.0841\textsuperscript{\scriptsize $\pm$0.0104}}
& 0.1982\textsuperscript{\scriptsize $\pm$0.0335}
& \textbf{0.0763\textsuperscript{\scriptsize $\pm$0.1265}} \\

\midrule

\multirow{4}{*}{\rotatebox{45}{\small \textbf{GMLE}}} 
& MSE~$(\downarrow)$ 
& 0.0034\textsuperscript{\scriptsize $\pm$0.0036}
& 0.0304\textsuperscript{\scriptsize $\pm$0.0063}
& \textbf{0.0327\textsuperscript{\scriptsize $\pm$0.0004}}
& 0.0078\textsuperscript{\scriptsize $\pm$0.0032}
& 0.8878\textsuperscript{\scriptsize $\pm$0.0545}
& 0.0221\textsuperscript{\scriptsize $\pm$0.0022}
& 0.0588\textsuperscript{\scriptsize $\pm$0.0274}
& 0.0079\textsuperscript{\scriptsize $\pm$0.0034}
& 0.0251\textsuperscript{\scriptsize $\pm$0.0068}
& 0.1012\textsuperscript{\scriptsize $\pm$0.1045} \\

& NLL~$(\downarrow)$ 
& -2.1598\textsuperscript{\scriptsize $\pm$0.3487}
& -0.2592\textsuperscript{\scriptsize $\pm$0.1759}
& -0.3893\textsuperscript{\scriptsize $\pm$0.0593}
& -0.9513\textsuperscript{\scriptsize $\pm$0.0313}
& 0.9812\textsuperscript{\scriptsize $\pm$0.5706}
& 0.2280\textsuperscript{\scriptsize $\pm$0.0988}
& -1.4809\textsuperscript{\scriptsize $\pm$0.2681}
& -1.1007\textsuperscript{\scriptsize $\pm$0.2740}
& -0.4129\textsuperscript{\scriptsize $\pm$0.1664}
& -0.4716\textsuperscript{\scriptsize $\pm$0.0933} \\

& CRPS~$(\downarrow)$ 
& 0.3003\textsuperscript{\scriptsize $\pm$0.0063}
& 0.0937\textsuperscript{\scriptsize $\pm$0.0091}
& 0.1031\textsuperscript{\scriptsize $\pm$0.0008}
& 0.0755\textsuperscript{\scriptsize $\pm$0.0043}
& 0.8181\textsuperscript{\scriptsize $\pm$0.0156}
& 0.1353\textsuperscript{\scriptsize $\pm$0.0123}
& 0.0981\textsuperscript{\scriptsize $\pm$0.0077}
& 0.0964\textsuperscript{\scriptsize $\pm$0.0079}
& 0.0816\textsuperscript{\scriptsize $\pm$0.0075}
& 0.2087\textsuperscript{\scriptsize $\pm$0.0597}\\

& ECE~$(\downarrow)$ 
& 0.4727\textsuperscript{\scriptsize $\pm$0.0864}
& 0.4084\textsuperscript{\scriptsize $\pm$0.0315}
& 0.5144\textsuperscript{\scriptsize $\pm$0.0055}
& 0.4457\textsuperscript{\scriptsize $\pm$0.0378}
& 0.6977\textsuperscript{\scriptsize $\pm$0.0580}
& 0.8824\textsuperscript{\scriptsize $\pm$0.0201}
& 0.9066\textsuperscript{\scriptsize $\pm$0.0068}
& 0.8609\textsuperscript{\scriptsize $\pm$0.0403}
& 0.2060\textsuperscript{\scriptsize $\pm$0.0340}
& 0.7531\textsuperscript{\scriptsize $\pm$0.1785} \\

\midrule

\multirow{4}{*}{\rotatebox{45}{\small \textbf{DE}}} 
& MSE~$(\downarrow)$ 
& 0.0071\textsuperscript{\scriptsize $\pm$0.0127}
& 0.0303\textsuperscript{\scriptsize $\pm$0.0061}
& \textbf{0.0327\textsuperscript{\scriptsize $\pm$0.0005}}
& \textbf{0.0059\textsuperscript{\scriptsize $\pm$0.0025}}
& 0.6881\textsuperscript{\scriptsize $\pm$0.2989}
& 0.0139\textsuperscript{\scriptsize $\pm$0.0023}
& 0.0536\textsuperscript{\scriptsize $\pm$0.0372}
& 0.0034\textsuperscript{\scriptsize $\pm$0.0005}
& 0.0251\textsuperscript{\scriptsize $\pm$0.0069}
& 0.0362\textsuperscript{\scriptsize $\pm$0.0191} \\

& NLL~$(\downarrow)$ 
& \textbf{-2.5632\textsuperscript{\scriptsize $\pm$0.3694}}
& -0.2633\textsuperscript{\scriptsize $\pm$0.1636}
& -0.3905\textsuperscript{\scriptsize $\pm$0.0577}
& -0.8922\textsuperscript{\scriptsize $\pm$0.0712}
& 2.3523\textsuperscript{\scriptsize $\pm$1.7854}
& 0.3432\textsuperscript{\scriptsize $\pm$0.0877}
& -1.6891\textsuperscript{\scriptsize $\pm$0.2163}
& 0.1531\textsuperscript{\scriptsize $\pm$0.2809}
& -0.4140\textsuperscript{\scriptsize $\pm$0.1662}
& 0.1456\textsuperscript{\scriptsize $\pm$0.0278} \\

& CRPS~$(\downarrow)$ 
& 0.0180\textsuperscript{\scriptsize $\pm$0.0099}
& 0.0945\textsuperscript{\scriptsize $\pm$0.0095}
& 0.1032\textsuperscript{\scriptsize $\pm$0.0008}
& \textbf{0.0494\textsuperscript{\scriptsize $\pm$0.0050}}
& 0.5005\textsuperscript{\scriptsize $\pm$0.1006}
& 0.0476\textsuperscript{\scriptsize $\pm$0.0081}
& 0.0376\textsuperscript{\scriptsize $\pm$0.0139}
& \textbf{0.0217\textsuperscript{\scriptsize $\pm$0.0027}}
& 0.0817\textsuperscript{\scriptsize $\pm$0.0076}
& 0.0809\textsuperscript{\scriptsize $\pm$0.0278} \\

& ECE~$(\downarrow)$ 
& 0.4285\textsuperscript{\scriptsize $\pm$0.1073}
& 0.4012\textsuperscript{\scriptsize $\pm$0.0388}
& 0.5025\textsuperscript{\scriptsize $\pm$0.0055}
& 0.2717\textsuperscript{\scriptsize $\pm$0.0648}
& 0.7919\textsuperscript{\scriptsize $\pm$0.0632}
& 0.5832\textsuperscript{\scriptsize $\pm$0.1622}
& 0.5451\textsuperscript{\scriptsize $\pm$0.2481}
& 0.5916\textsuperscript{\scriptsize $\pm$0.1966}
& 0.2015\textsuperscript{\scriptsize $\pm$0.0339}
& 0.5983\textsuperscript{\scriptsize $\pm$0.1189} \\

\midrule

\multirow{4}{*}{\rotatebox{45}{\small \textbf{MCD}}} 
& MSE~$(\downarrow)$ 
& 0.0094\textsuperscript{\scriptsize $\pm$0.0084}
& 0.0418\textsuperscript{\scriptsize $\pm$0.0058}
& \textbf{0.0327\textsuperscript{\scriptsize $\pm$0.0005}}
& 0.0070\textsuperscript{\scriptsize $\pm$0.0040}
& 0.6464\textsuperscript{\scriptsize $\pm$0.1363}
& 0.0145\textsuperscript{\scriptsize $\pm$0.0087}
& 0.0670\textsuperscript{\scriptsize $\pm$0.0276}
& 0.0082\textsuperscript{\scriptsize $\pm$0.0070}
& 0.0252\textsuperscript{\scriptsize $\pm$0.0068}
& \textbf{0.0102\textsuperscript{\scriptsize $\pm$0.0034}} \\

& NLL~$(\downarrow)$ 
& 1.7845\textsuperscript{\scriptsize $\pm$0.1469}
& \textbf{-0.2618\textsuperscript{\scriptsize $\pm$0.1704}}
& -0.3889\textsuperscript{\scriptsize $\pm$0.0598}
& -0.7977\textsuperscript{\scriptsize $\pm$0.1874}
& 6.7464\textsuperscript{\scriptsize $\pm$6.8760}
& 0.0394\textsuperscript{\scriptsize $\pm$0.0028}
& -0.6656\textsuperscript{\scriptsize $\pm$0.1664}
& 0.5466\textsuperscript{\scriptsize $\pm$0.2362}
& -0.4053\textsuperscript{\scriptsize $\pm$0.1577}
& 0.4916\textsuperscript{\scriptsize $\pm$0.5748} \\

& CRPS~$(\downarrow)$ 
& 0.0616\textsuperscript{\scriptsize $\pm$0.0090}
& 0.1168\textsuperscript{\scriptsize $\pm$0.0073}
& 0.1033\textsuperscript{\scriptsize $\pm$0.0008}
& 0.0525\textsuperscript{\scriptsize $\pm$0.0096}
& 0.5242\textsuperscript{\scriptsize $\pm$0.0746}
& 0.0458\textsuperscript{\scriptsize $\pm$0.0131}
& 0.0433\textsuperscript{\scriptsize $\pm$0.0094}
& 0.0277\textsuperscript{\scriptsize $\pm$0.0090}
& 0.0820\textsuperscript{\scriptsize $\pm$0.0074}
& \textbf{0.0354\textsuperscript{\scriptsize $\pm$0.0063}}\\

& ECE~$(\downarrow)$ 
& 0.1977\textsuperscript{\scriptsize $\pm$0.0417}
& 0.3613\textsuperscript{\scriptsize $\pm$0.0345}
& 0.5056\textsuperscript{\scriptsize $\pm$0.0071}
& 0.2810\textsuperscript{\scriptsize $\pm$0.0727}
& 0.8318\textsuperscript{\scriptsize $\pm$0.0620}
& 0.3476\textsuperscript{\scriptsize $\pm$0.0597}
& 0.4348\textsuperscript{\scriptsize $\pm$0.1056}
& 0.4489\textsuperscript{\scriptsize $\pm$0.1319}
& 0.2013\textsuperscript{\scriptsize $\pm$0.0315}
& 0.4249\textsuperscript{\scriptsize $\pm$0.0797}\\
\midrule

\multirow{4}{*}{\rotatebox{45}{\small \textbf{DER}}} 
& MSE~$(\downarrow)$ 
& 0.1694\textsuperscript{\scriptsize $\pm$0.0033}
& 0.0310\textsuperscript{\scriptsize $\pm$0.0059}
& 0.0328\textsuperscript{\scriptsize $\pm$0.0006}
& 0.0154\textsuperscript{\scriptsize $\pm$0.0028}
& 1.0938\textsuperscript{\scriptsize $\pm$0.0177}
& 0.0331\textsuperscript{\scriptsize $\pm$0.0004}
& \textbf{0.0252\textsuperscript{\scriptsize $\pm$0.0001}}
& 0.0296\textsuperscript{\scriptsize $\pm$0.0001}
& 0.0253\textsuperscript{\scriptsize $\pm$0.0068}
& 0.0913\textsuperscript{\scriptsize $\pm$0.0062} \\

& NLL~$(\downarrow)$ 
& -2.5026\textsuperscript{\scriptsize $\pm$0.1637}
& -0.2425\textsuperscript{\scriptsize $\pm$0.1352}
& -0.3878\textsuperscript{\scriptsize $\pm$0.0123}
& -0.9593\textsuperscript{\scriptsize $\pm$0.2032}
& 3.1940\textsuperscript{\scriptsize $\pm$0.1607}
& \textbf{-2.2316\textsuperscript{\scriptsize $\pm$0.5052}}
& \textbf{-2.1217\textsuperscript{\scriptsize $\pm$1.3504}}
& \textbf{-2.2585\textsuperscript{\scriptsize $\pm$1.0695}}
& \textbf{-0.6008\textsuperscript{\scriptsize $\pm$0.0965}}
&-1.7055\textsuperscript{\scriptsize $\pm$1.7941} \\

& CRPS~$(\downarrow)$ 
& 0.9803\textsuperscript{\scriptsize $\pm$0.05754}
& 0.2111\textsuperscript{\scriptsize $\pm$0.0458}
& 0.8011\textsuperscript{\scriptsize $\pm$0.1905}
& 0.9438\textsuperscript{\scriptsize $\pm$0.3740}
& 0.9357\textsuperscript{\scriptsize $\pm$0.2642}
& 0.1099\textsuperscript{\scriptsize $\pm$0.0055}
& 0.6716\textsuperscript{\scriptsize $\pm$0.0978}
& 0.0814\textsuperscript{\scriptsize $\pm$0.0070}
& 0.7994\textsuperscript{\scriptsize $\pm$0.0004}
& 0.7217\textsuperscript{\scriptsize $\pm$0.0319} \\

& ECE~$(\downarrow)$ 
& \textbf{0.0431\textsuperscript{\scriptsize $\pm$0.0443}}
& 0.1208\textsuperscript{\scriptsize $\pm$0.0296}
& \textbf{0.0001\textsuperscript{\scriptsize $\pm$0.0000}}
& 0.0012\textsuperscript{\scriptsize $\pm$0.0003}
& \textbf{0.0113\textsuperscript{\scriptsize $\pm$0.0030}}
& 0.8621\textsuperscript{\scriptsize $\pm$0.0179}
& 0.1484\textsuperscript{\scriptsize $\pm$0.1436}
& 0.3441\textsuperscript{\scriptsize $\pm$0.1347}
& \textbf{0.0125\textsuperscript{\scriptsize $\pm$0.0024}}
& 0.0582\textsuperscript{\scriptsize $\pm$0.1119} \\
\midrule

\multirow{4}{*}{\rotatebox{45}{\small \textbf{SDE-Net}}} 
& MSE~$(\downarrow)$ 
& 0.0242\textsuperscript{\scriptsize $\pm$0.0190}
& 0.0348\textsuperscript{\scriptsize $\pm$0.0071}
& 0.0328\textsuperscript{\scriptsize $\pm$0.0005}
& 0.0199\textsuperscript{\scriptsize $\pm$0.0035}
& 1.6981\textsuperscript{\scriptsize $\pm$0.2651}
& 0.0068\textsuperscript{\scriptsize $\pm$0.0023}
& 0.0239\textsuperscript{\scriptsize $\pm$0.0032}
& 0.0170\textsuperscript{\scriptsize $\pm$0.0017}
& 0.0255\textsuperscript{\scriptsize $\pm$0.0071}
& 0.0550\textsuperscript{\scriptsize $\pm$0.0213} \\

& NLL~$(\downarrow)$ 
& -2.0640\textsuperscript{\scriptsize $\pm$0.2909}
& -0.1729\textsuperscript{\scriptsize $\pm$0.1785}
& -0.3822\textsuperscript{\scriptsize $\pm$0.0578}
& -1.5697\textsuperscript{\scriptsize $\pm$0.2236}
& -1.1297\textsuperscript{\scriptsize $\pm$0.48276}
& 0.9260\textsuperscript{\scriptsize $\pm$0.1421}
& -1.1261\textsuperscript{\scriptsize $\pm$0.1385}
& -1.9123\textsuperscript{\scriptsize $\pm$0.2048}
& -0.3875\textsuperscript{\scriptsize $\pm$0.1544}
& -1.7786\textsuperscript{\scriptsize $\pm$0.7054} \\

& CRPS~$(\downarrow)$ 
& 0.8119\textsuperscript{\scriptsize $\pm$0.0185}
& 0.4041\textsuperscript{\scriptsize $\pm$0.0020}
& 0.4048\textsuperscript{\scriptsize $\pm$0.0020}
& 0.7437\textsuperscript{\scriptsize $\pm$0.0129}
& 0.8338\textsuperscript{\scriptsize $\pm$0.0743}
& 0.9017\textsuperscript{\scriptsize $\pm$0.0803}
& 0.8534\textsuperscript{\scriptsize $\pm$0.0731}
& 0.9478\textsuperscript{\scriptsize $\pm$0.0880}
& 0.3771\textsuperscript{\scriptsize $\pm$0.0143}
& 0.9114\textsuperscript{\scriptsize $\pm$0.0745} \\

& ECE~$(\downarrow)$ 
& 0.0475\textsuperscript{\scriptsize $\pm$0.0340}
& \textbf{0.0632\textsuperscript{\scriptsize $\pm$0.0086}}
& 0.0673\textsuperscript{\scriptsize $\pm$0.0007}
& \textbf{0.0285\textsuperscript{\scriptsize $\pm$0.0038}}
& 0.5171\textsuperscript{\scriptsize $\pm$0.0396}
& 0.0129\textsuperscript{\scriptsize $\pm$0.0066}
& 0.1591\textsuperscript{\scriptsize $\pm$0.0034}
& 0.0204\textsuperscript{\scriptsize $\pm$0.0034}
& 0.0424\textsuperscript{\scriptsize $\pm$0.0119}
& 0.0803\textsuperscript{\scriptsize $\pm$0.0326} \\

\bottomrule

\end{tabular}}
\label{tab:results}
\begin{minipage}{\textwidth}
\footnotesize
\textbf{Note:} (↑) higher is better; (↓) lower is better.
\vspace{-5mm}
\end{minipage}
\end{table*}

\begin{figure*}[ht!]
\centering
\includegraphics[width=1.0\textwidth]{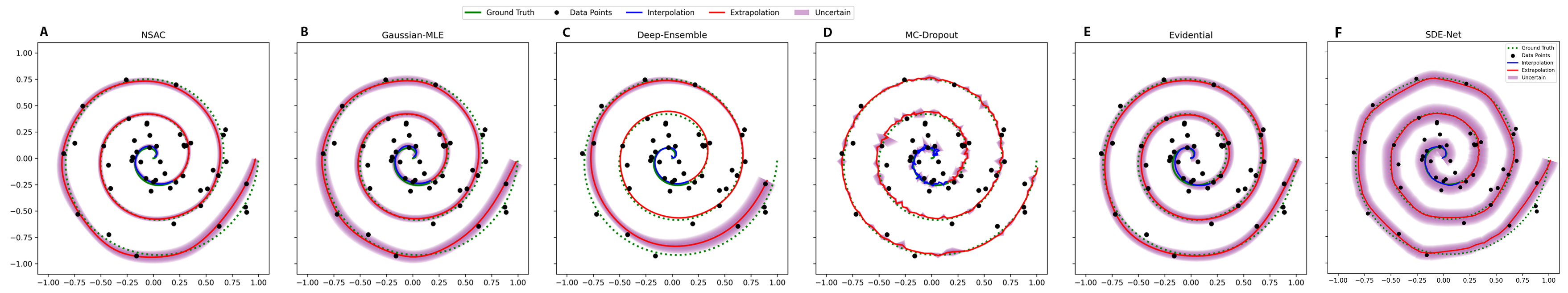}
\caption{Spiral trajectory with uncertainty estimates: \textbf{(A)} NSAC; \textbf{(B)} GMLE; \textbf{(C)} DE; \textbf{(D)} MCD; \textbf{(E)} DER; and \textbf{(F)} SDE-Net. NSAC and DER provide more compact uncertainty tubes with smooth mean accuracy.} 
\vspace{-3mm}
\label{fig:spiral_baseline}
\end{figure*}

\subsection{Irregular CT Function Approximation}
In the first experiment, we evaluate NSAC’s ability to approximate irregular CT function with uncertainty estimates. We curated the irregular spiral dataset following the procedure in~\cite{chen2023contiformer}, which generated 300 2D spiral trajectories, each sampled at 150 equally spaced time points. To introduce irregularity, we randomly selected 50 points from each trajectory without replacement. We then visualize the resulting interpolation and extrapolation. The qualitative results with uncertainty estimates are shown in Figure~\ref{fig:spiral_baseline} and quantitative results are reported in Table~\ref{tab:results}. \\
NSAC achieved the strongest predictive performance, with the lowest MSE of 0.0002, followed by DE~(0.0071) and MCD~(0.0094). It also achieved the best CRPS of 0.0095, outperforming DE~(0.0180), MCD~(0.0616), and SDE-Net~(0.8119), confirming its superior accuracy and distributional sharpness. In NLL, DE ranked first with -2.5632, followed closely by NSAC~(-2.3483). The DER exhibited severe distributional failure, with the CRPS degrading to 0.9801 despite a favorable NLL~(-2.5026). With respect to calibration, DER and SDE-Net achieved the lowest ECE (0.0431 and 0.0475, respectively), whereas NSAC achieved moderate calibration of 0.1631.\\
\textbf{Ablation Decomposition.} We also perform an ablation study on spirals to examine the effect of regularization on uncertainty decomposition. The qualitative visualization of the decomposition is presented in Figure~\ref{fig:spiral_ablation}, and the quantitative metrics (MSE \& MAE) are presented in Table~\ref{tab:ablation_results}. This confirms the functional role of $\mathcal{L_{\text{reg}}}$. With $\mathcal{L_{\text{reg}}}$, the NSAC successfully decomposes uncertainty into distinct aleatoric and epistemic components. Without $\mathcal{L_{\text{reg}}}$, epistemic uncertainty is suppressed while aleatoric uncertainty is retained, and predictive accuracy is preserved with near-identical MSEs and MAEs across interpolation (MSE: 0.0002, MAE: 0.0143 vs 0.0117) and extrapolation (MSE: 0.0017 vs. 0.0012, MAE: 0.0304 vs. 0.0272) regimes. 
\begin{figure}[t]
\centering
\includegraphics[width=0.98\columnwidth]{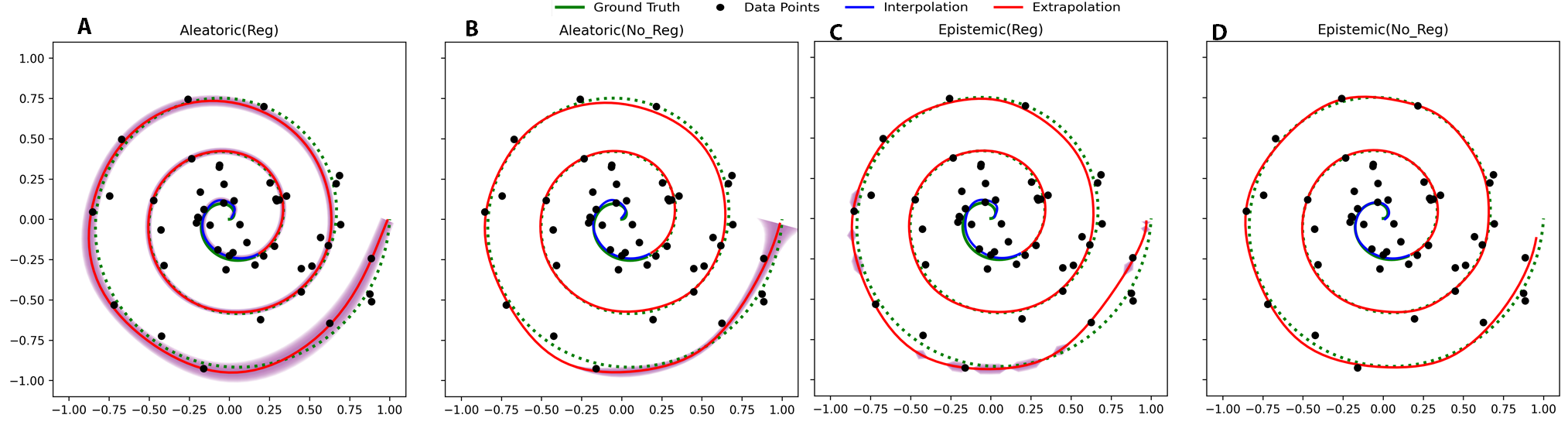}
\caption{Ablation Decomposition of NSAC uncertainty estimates: \textbf{(A)} Aleatoric uncertainty with regularizer; \textbf{(B)} Aleatoric uncertainty without regularizer; \textbf{(C)} Epistemic uncertainty with regularizer; \textbf{(D)} Epistemic uncertainty without regularizer.}
\label{fig:spiral_ablation}
\end{figure}

\begin{table}[t]
\centering
\caption{Ablation decomposition results}
\resizebox{0.85\columnwidth}{!}{%
\begin{tabular}{lcccc}
\toprule
\multirow{2}{*}{Metric} & \multicolumn{2}{c}{Aleatoric} & \multicolumn{2}{c}{Epistemic} \\
\cmidrule(lr){2-3}  \cmidrule(lr){4-5}
       & $\mathcal{L}_{\text{reg}}$ & W/O $\mathcal{L}_{\text{reg}}$ & $\mathcal{L}_{\text{reg}}$ & W/O $\mathcal{L}_{\text{reg}}$ \\
\midrule
\multicolumn{5}{c}{Task = \textit{Interpolation}} \\
\midrule
MSE ($\downarrow$) & 0.0002 & 0.0002 & 0.0002 & 0.0002 \\
MAE ($\downarrow$) & 0.0143 & 0.0117 & 0.0119 & 0.0058 \\
\midrule
\multicolumn{5}{c}{Task = \textit{Extrapolation}} \\
\midrule
MSE ($\downarrow$) & 0.0017 & 0.0012 & 0.0018 & 0.0012 \\
MAE ($\downarrow$) & 0.0304 & 0.0272 & 0.0300 & 0.0237 \\
\bottomrule
\end{tabular}}
\vspace{-5mm}
\label{tab:ablation_results}
\end{table}

\subsection{Multivariate Regression}
In the second experiment, we evaluate all the methods on a multivariate regression task. Specifically, we utilized two well-known regression benchmark datasets: (i) Boston Housing; and (ii) Kin8nm. Prior to training, both datasets were normalized using \texttt{MinmaxScalar}.\\
\textbf{Boston:} The Boston Housing dataset has 506 samples of homes in Boston suburbs, each with 13 features, such as crime rate, average rooms, and accessibility to highways, with the target being the median home value. NSAC achieves the lowest MSE of 0.030, outperforming DE~(0.0303) and the GMLE~(0.0304), while improving over MCD~(0.0418) and SDE-Net~(0.0348). The DER attains a competitive MSE but exhibits severe distributional distortion, with the CRPS increasing to 0.211, which is double those of the NSAC~(0.0940) and DE~(0.0945). NLL values are tightly clustered for NSAC~(-0.2665), DE~(-0.2633), and MCD~(-0.2618). In terms of calibration, SDE-Net achieves the lowest ECE of 0.0632, followed by DER~(0.1208), while NSAC~(0.3974) is comparable to DE~(0.4012).\\
\textbf{Kin8nm:} The Kin8nm dataset contains a total of 8192 samples, where each sample contains 8 numerical input features representing physical measurements and a single continuous target variable. In this test, the prediction accuracy is nearly saturated, with all methods achieving MSEs between 0.0327 and 0.0328, making the distributional quality the primary differentiator. NLL is also effectively indistinguishable for NSAC~(-0.3870), GMLE~(-0.3893), DE~(-0.3905), MCD~(-0.3889), and DER~(-0.3878). NSAC achieves a CRPS of 0.1029, matching those of the GMLE, DE, and MCD, whereas DER and SDE-Net substantially worsens despite having comparable MSE. The near-zero ECE of the DER likely reflects evidential overfitting rather than genuine calibration, which is consistent with its elevated CRPS variance. NSAC has an ECE of 0.3085.
\begin{figure}[t]
\centering
\includegraphics[width=1.0\columnwidth]{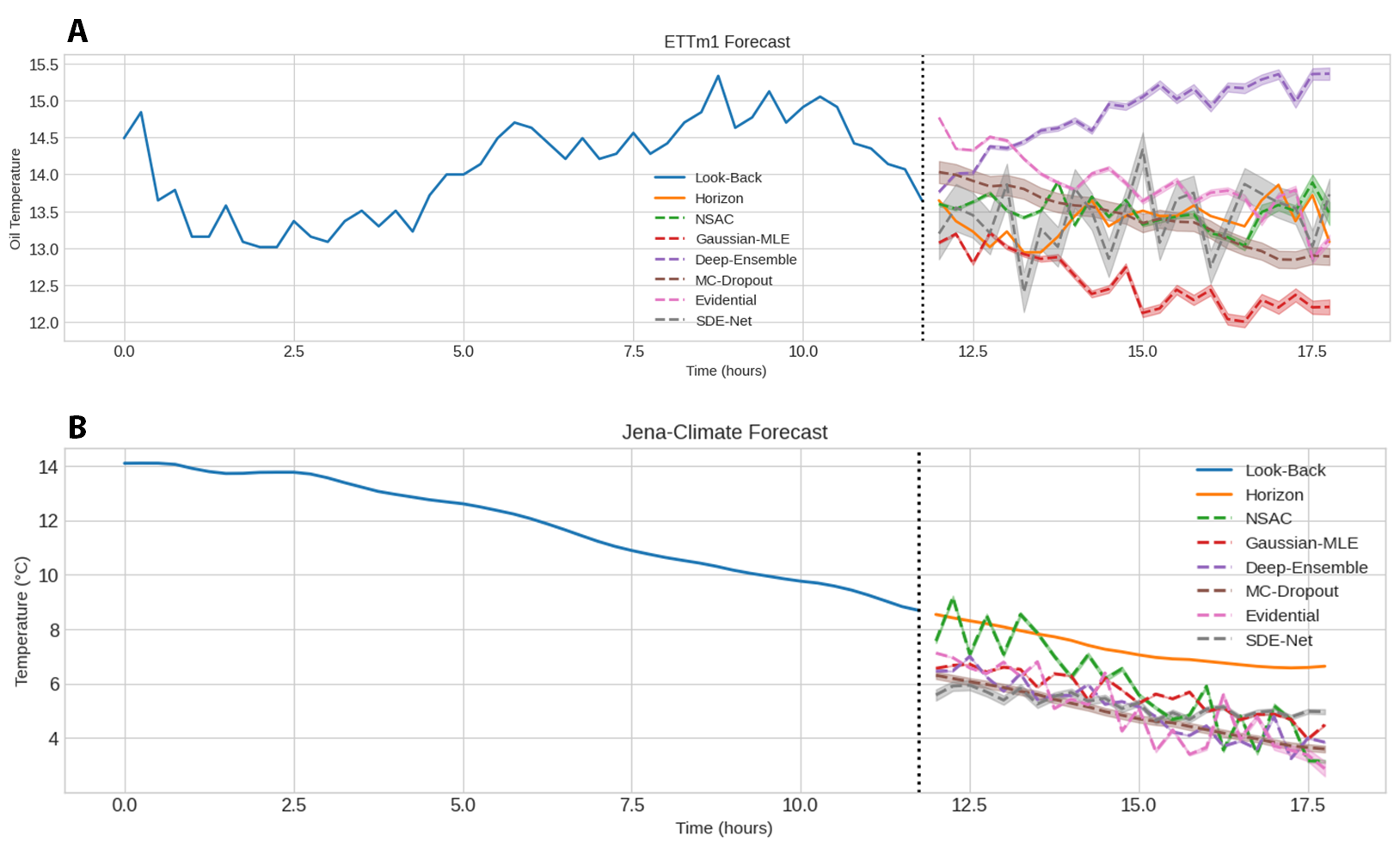}
\caption{Visualization of forecast projections with uncertainty estimates: (\textbf{A}) ETTm1; and (\textbf{B}) Jena-Climate.}
\label{fig:lrm_exps}
\vspace{-5mm}
\end{figure}
\subsection{Multivariate Long-range Forecasting}
In the third experiment, we evaluate the multivariate long-range forecasting capability of NSAC. Two widely used benchmark datasets are utilized: (i) ETTm1; and (ii) Jena-Climate. Prior to training, both datasets were transformed using \texttt{MinMaxScaler}, and predictions were inverse-transformed back for post-training evaluation. Each dataset was split into training (80\%), validation (10\%), and test (10\%) folds. Figure~\ref{fig:lrm_exps} illustrates the forecasting projections with uncertainty tubes.\\
\textbf{ETTm1:} The ETTm1 dataset consists of 7 time-dependent features sampled at 15-minute intervals. For forecasting, the model was conditioned on the preceding 12 hours (48 look-back) to predict the subsequent 6 hours (24 horizons). Post-hoc baselines achieve stronger predictive accuracy, with DE attaining the lowest MSE of 0.0059, followed by MCD (0.0070), GMLE (0.0078), and NSAC (0.0099). However, the NSAC achieves the best NLL of -1.3816, outperforming all the baselines. GMLE and DE attain the lowest CRPS values of 0.0755 and 0.0494, respectively, while NSAC achieves 0.1037. Similar to Kin8nm test, the DER’s near-zero ECE of 0.0012$\pm$0.0003 likely reflects evidential overfitting, whereas the NSAC records an ECE of 0.3970 . Figure~\ref{fig:lrm_exps}(A) indicates that NSAC forecasting trajectories align more closely with the ground truth.\\
\textbf{Jena-Climate:} The Jena-Climate dataset contains 14 atmospheric variables recorded in Jena, Germany, at 10-minute intervals. The model was trained using the past 8 hours (48 look-back ) to forecast the following 4 hours (24 horizon). NSAC achieves the lowest MSE of 0.1675, substantially outperforming MCD~(0.6464), DE~(0.6881), DER~(1.0938), and SDE-Net~(1.6981). It also achieves the most favorable NLL of -1.9114 and the lowest CRPS of 0.1422, while all the baselines exhibit severely degraded density estimation. Figure~\ref{fig:lrm_exps}(B) shows that NSAC produces more accurate forecasting trajectories than competing methods.
\begin{figure*}[ht!]
\centering
\includegraphics[width=0.85\textwidth]{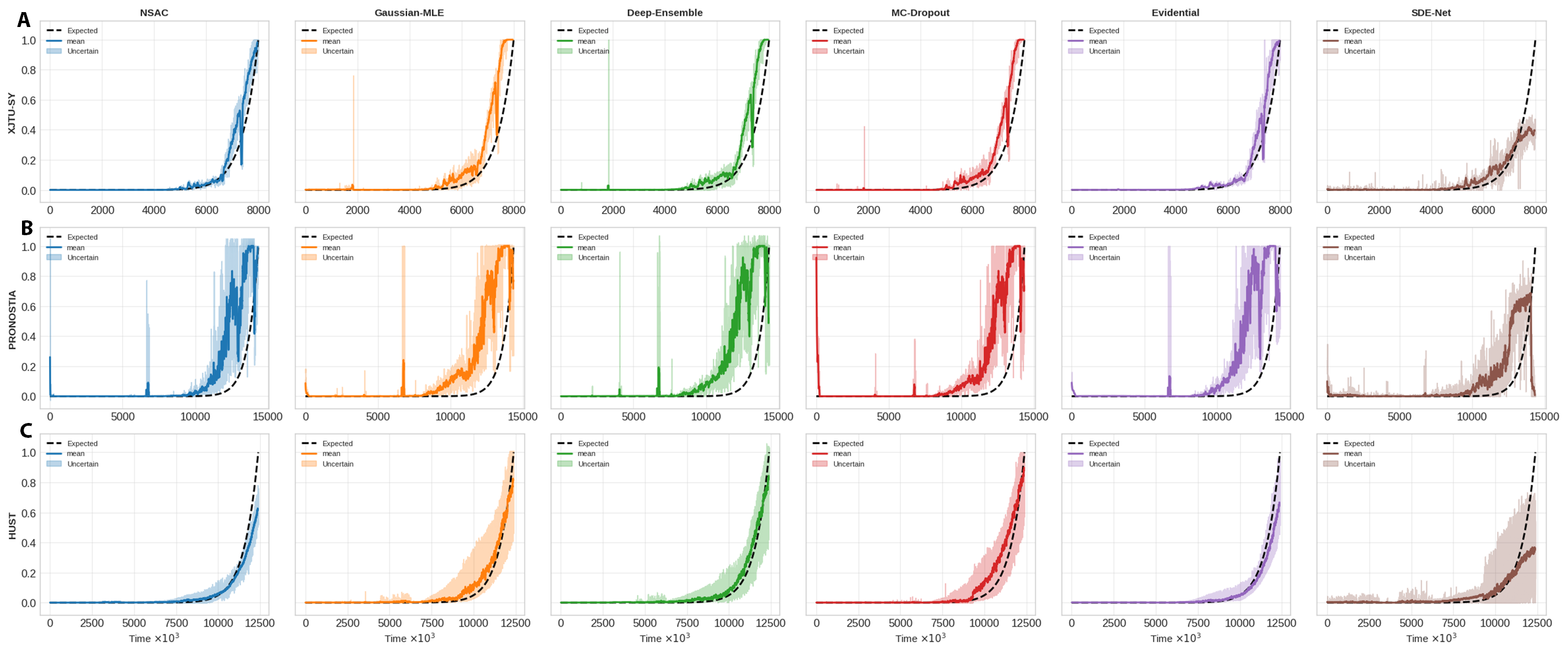}
\caption{Visualization of degradation trajectory with uncertainty estimates. (\textbf{A}) XJTU-SY; (\textbf{B}) PRONOSTIA; and (\textbf{C}) HUST.} 
\vspace{-4mm}
\label{fig:rul}
\end{figure*}
\subsection{Industry~4.0}
In this fourth experiment, we evaluate the ability of the NSAC to model physical dynamics through the task of bearing degradation estimation. Bearing degradation prediction is a classical problem in industrial engineering and is typically addressed using data-driven approaches. However, such methods often suffer from limited generalizability beyond their training regimes, while their large model sizes make them impractical for deployment in resource-constrained environments. The objective of this experiment is therefore to learn degradation dynamics from a single operating condition and assess zero-shot generalization to entirely unseen datasets. It also aims to assess the distributional generalization of uncertainty estimates and prediction accuracy across datasets.\\
We conduct experiments on three widely used benchmark datasets: (i) XJTU-SY~\cite{wang2018xjtu}, (ii) PRONOSTIA~\cite{nectoux2012pronostia}, and (iii) HUST~\cite{thuan2023hust}. Training is performed using the first three bearing from the XJTU-SY dataset, while evaluation is conducted in a zero-shot manner on the remaining datasets. Prior to training, the raw vibration signals are transformed into physically meaningful features following \cite{razzaq2025carle}, and nonlinear degradation labels are generated using the procedure described in \cite{razzaq2025developingdistanceawareuncertaintyquantification}. This formulation reduces the problem to a long-horizon regression task. Figure~\ref{fig:rul} presents the qualitative results.\\
\textbf{XJTU-SY (in-distribution):} NSAC achieves the lowest MSE of 0.0048, outperforming SDE-Net~(0.0068), MCD~(0.0145), DE~(0.0139), GMLE~(0.0221), and DER~(0.0331). Although the DER attains the most favorable NLL of -2.2316, the NSAC achieves a competitive value of -0.0219, surpassing those of the GMLE, DE, MCD, and SDE-Net. NSAC further records the lowest CRPS of 0.0261 and an ECE of 0.1404. Figure~\ref{fig:rul}(A) confirms NSAC's superior probabilistic accuracy and calibration relative to those of all the baselines on XJTU-SY.\\
\textbf{PRONOSTIA (out-of-distribution):} Under a distributional shift, the NSAC achieves the lowest CRPS~(0.0310) and ECE~(0.1337), demonstrating strong generalizability in terms of probabilistic estimation. Although DER achieves the lowest MSE~(0.0252) and NLL~(-2.1217), its substantially higher CRPS~(0.6716) suggests significant distributional distortion and poor uncertainty quality. Figure~\ref{fig:rul}(B) illustrates the degradation trajectory with uncertainty tubes for PRONOSTIA.\\
\textbf{HUST (out-of-distribution):} NSAC achieves the lowest MSE~(0.0033) and ECE~(0.0841). DE records the lowest CRPS~(0.0217), marginally outperforming NSAC~(0.0231). DER again achieves the most favorable NLL~(-2.2585), followed by NSAC~(-1.5889). Figure~\ref{fig:rul}(C) illustrates the degradation trajectory with uncertainty tubes for HUST.

\subsection{Autonomous Vehicle}
In the fifth experiment, we evaluate the ability of NSAC to function as a step-wise controller for autonomous vehicles. To perform this evaluation, we utilize two widely used autonomous driving simulation environments: (i) the Udacity Self-Driving Car simulator and (ii) OpenAI CarRacing. The objective is to develop an end-to-end controller that takes an input image stream and predicts steering actions based on the observed scene. Specifically, the task requires learning the relationship between the road horizon and the corresponding steering commands.\\
\textbf{Udacity Simulator:} For the Udacity simulator, the training dataset is generated by manually driving the vehicle for 30 minutes while the front-mounted camera input and corresponding steering angles are recorded. This produces a dataset of 15647 RGB images of size \(320\times160\times3\) with corresponding steering angles. We adopt the preprocessing pipeline and architecture introduced in \cite{shibuya_car_behavioral_cloning}, replacing the latent dense layers with NSAC or baseline. The quantitative results are summarized in Table~\ref{tab:results}. All the models converge to nearly identical predictive accuracies, with MSE values ranging from 0.0249 to 0.0255. DER achieves the lowest NLL of -0.6008, followed by DE~(-0.4140), the GMLE~(-0.4129), and the NSAC~(-0.4122). NSAC attains the lowest CRPS of 0.0814, marginally outperforming the GMLE~(0.0816) and DE~(0.0817) and substantially surpassing the DER~(0.7994). DER achieves the lowest ECE of 0.0125, followed by SDE-Net~(0.0424), while NSAC records an ECE of 0.1982.\\
\textbf{OpenAI-CarRacing:} For OpenAI CarRacing, the training dataset is generated using a PPO agent trained for two million steps to autonomously drive the vehicle while recording observations and corresponding control actions, including the steering angle, throttle, and brake. This results in a dataset of 48102 grayscale images of size $84 \times 84$ paired with control actions. We develop an end-to-end neural architecture consisting of convolutional layers for spatial feature extraction followed by NSAC or baseline methods for action prediction in a regression setting. The quantitative results are reported in Table~\ref{tab:results}. NSAC achieves the most favorable NLL of -1.6450 and a competitive ECE of 0.0763, with an MSE of 0.0154. The MCD attains the lowest MSE of 0.0102 and the lowest CRPS of 0.0354. A detailed closed-loop analysis with trajectory analysis is presented in Appendix~\ref{appendix:closed-loop}.

\begin{figure*}[t]
\centering
\includegraphics[width=0.85\textwidth]{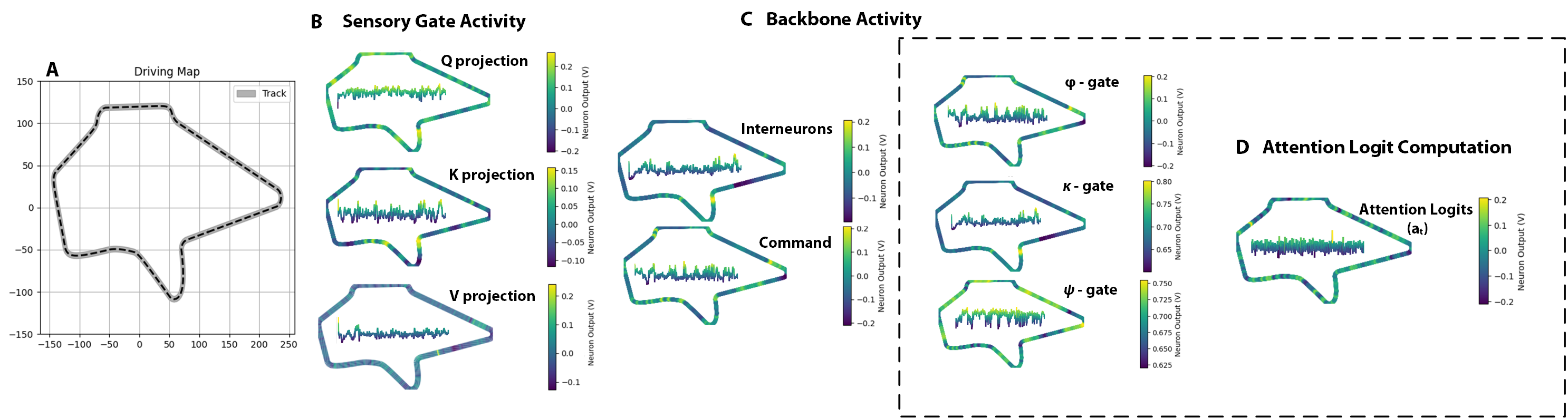}
\caption{Intuitive visualization of NSAC cell activity during driving. (\textbf{A}) Test drive map; (\textbf{B}) Sensory neurons output (Centroid) for $q$, $k$, $v$ projections; (\textbf{C}) Backbone activity with interneurons (Centroid), Command neurons (Centroid) and OU coefficients ($\kappa$, $\phi$, and $\psi$) outputs (\textbf{D}) Attention logit ($a_t$) computation. Internal plots are actual neuron activities.}\vspace{-3mm}
\label{fig:interpretability}
\end{figure*}

\subsection{NSAC Carries Interpretability}
Interpretability refers to explaining AI decisions in a human-understandable manner through the analysis of cell-level potentials. We conduct a cell-level analysis of the NSAC on the OpenAI-CarRacing environment using a the test track shown in Figure~\ref{fig:interpretability}(A). Specifically, we trace sensory projections ($q$, $k$, $v$) through the backbone, from inter-command neurons to the output heads $\kappa$, $\phi$, and $\psi$ and record cell potentials at each point along the test track. The output head $\phi$ varies within [-0.2, 0.2], with values below -0.10 associated with the left steering turn, values above 0.10 associated with the right steering turn, and intermediate values corresponding to straight. Similarly, $\kappa$ remains within [0.65, 0.8], where values below 0.7 indicate left turns and values above 0.75 indicate right turns. The head $\psi$ lies within [0.625, 0.750], with values below 0.685 corresponding to left turns and values above 0.715 corresponding to right turns. The attention logits vary within [-0.2, 0.2], where values above 0.1 indicate right steering, values below -0.1 indicate left steering, and intermediate values correspond to straight driving.

\section{Discussion \& Limitations}\label{sec:discuss}
NSAC frequently achieved favorable NLL and CRPS values across several benchmarks, suggesting that stochastic attention can improve aspects of probabilistic prediction beyond point-estimation performance alone. The moderate ECE observed in some settings is consistent with the stochastic inductive bias of the OU formulation, where locally sharper uncertainty estimates do not always translate into exact marginal coverage. The empirical results indicate that gains in probabilistic metrics do not always coincide with improvements in calibration, highlighting an inherent trade-off between distributional sharpness and calibration quality. DER further shows evidential collapse, with near-zero ECE but degraded CRPS on spiral and ETTm1, indicating overconfident predictions. \\
NSAC relies on Gaussian perturbations to generate surrogate OOD samples as a scalable, domain-agnostic mechanism to induce separation, which may not faithfully reflect semantic distribution shifts. Exploring alternatives such as adversarial or domain-specific, data-driven perturbations could be a promising direction for future work.\\
NSAC provides a model-integrated approach for incorporating uncertainty in CT attention architectures, which is particularly relevant in settings that require temporal continuity and tight architectural integration. NSAC is not intended as a replacement for established uncertainty quantification methods, but rather as a complementary mechanism that enables uncertainty to be modeled directly within the architecture without introducing instability.

\section{Conclusion}\label{sec:conclude}
In this work, we introduced Neuronal Stochastic Attention Circuit (NSAC), a novel biologically inspired continuous-time (CT) stochastic attention architecture that reformulates attention logits as the solution of an Ornstein-Uhlenbeck SDE (OU-SDE) modulated by nonlinear, input-dependent, interlinked gates derived from the \textit{C.~elegans} Neuronal Circuit Policies (NCPs) wiring mechanism. NSAC admits a closed-form forward pass under frozen-coefficient assumption, avoiding expensive numerical SDE solvers while enabling principled stochasticity within the attention logits. Logistic-normal normalization propagates this stochasticity to the attention weights, allowing the computation of attention scores and heteroscedastic probabilistic outputs. A two-term objective function combining Gaussian negative log-likelihood estimation with an epistemic-separation regularizer enables NSAC to jointly learn aleatoric and epistemic uncertainty estimates. Theoretically, we presented rigorous stability bounds and finite-time guarantees for both the SDE formulation and its closed-form approximation. Empirically, we evaluated NSAC across a diverse set of learning tasks and compared it with several state-of-the-art UQ baselines. Across all tasks, NSAC achieved competitive predictive performance and frequently attained favorable probabilistic metrics with informative uncertainty estimates while being interpretable at the neuronal-cell level.


\section*{Reproducibility Statement}
The code for reproducibility is available at  \url{https://github.com/itxwaleedrazzaq/neuronal_stochastic_attention_circuit}.






\bibliographystyle{plainnat}
\bibliography{references}

\appendix
\section*{Appendix}

\section{Proofs}
In this section we provide all proofs.
\subsection{Analyzing Closed-form Solution}

\subsubsection{Derivation of Closed-form Solution}\label{appendix:closed-form}
We derive the closed-form solution of NSAC. Throughout this derivation, the input-dependent gate parameters $\kappa(u)$, $\phi(u)$, and $\psi(u)$ are treated as piecewise constant within each discrete update interval (locally frozen coefficients~\cite{john1952integration}). Under this assumption, the system reduces to a linear SDE with constant coefficients on each interval, admitting an exact analytical solution. Let the attention logit $a_t$ evolve according to
\begin{equation}
da_t = \kappa(\phi - a_t)\,dt + \psi\,dW_t,
\end{equation}
Rewriting it in standard linear form,
\begin{equation}
da_t + \kappa a_t\,dt = \kappa \phi\,dt + \psi\,dW_t.
\end{equation}
Multiplying both sides by the integrating factor $e^{\kappa t}$ and applying It\^o's product rule yields
\begin{equation}
d\!\left(e^{\kappa t} a_t\right)
= \kappa \phi e^{\kappa t}\,dt + \psi e^{\kappa t}\,dW_t.
\end{equation}
Integrating from $0$ to $t$,
\begin{equation}
e^{\kappa t} a_t - a_0
= \kappa \phi \int_0^t e^{\kappa s}\,ds
+ \psi \int_0^t e^{\kappa s}\,dW_s. 
\end{equation}
Evaluating the deterministic integral:
\begin{equation}
\kappa \phi \int_0^t e^{\kappa s}\,ds = \phi\left(e^{\kappa t} - 1\right),  
\end{equation}
and dividing through by $e^{\kappa t}$ gives the closed-form solution
\begin{equation}
a_t = \phi + (a_0 - \phi)e^{-\kappa t}
+ \psi \int_0^t e^{-\kappa (t-s)}\,dW_s. 
\end{equation}
Taking expectations and using $\mathbb{E}[dW_t]=0$,
\begin{equation}
\mathbb{E}[a_t] = \phi + (a_0 - \phi)e^{-\kappa t}.
\end{equation}
The stochastic integral is the sole source of randomness. Applying It\^o isometry,
\begin{equation}
\mathrm{Var}(a_t)
= \psi^2 \int_0^t e^{-2\kappa (t-s)}\,ds
= \frac{\psi^2}{2\kappa}\left(1 - e^{-2\kappa t}\right).
\end{equation}
Thus the variance increases monotonically from $0$ and saturates at
\begin{equation}
\lim_{t \to \infty} \mathrm{Var}(a_t) = \frac{\psi^2}{2\kappa}.
\end{equation}
The mean reverts exponentially from the initial value $a_0$ toward the long-term mean $\phi$ at rate $\kappa$, while the variance approaches a stationary limit determined by the noise scale $\psi$ and the reversion strength $\kappa$.

\subsubsection{Frozen Coefficients Error Analysis}\label{appendix:frozen_coefficient}
The frozen-coefficients approximation may induce small errors relative to time-varying coefficients; here we now analyze this.

\begin{theorem}[Frozen-Coefficient Approximation Error]\label{theorem:frozen}
Let $\mu_{\mathrm{varying}}(t)$ and $\sigma_{\mathrm{varying}}^2(t)$ denote the mean and variance under time-varying coefficients, and let the frozen-coefficient approximations be $\mu_{\mathrm{frozen}}(t) = \phi_0 + (a_0 - \phi_0)e^{-\kappa_0 t}$, and $\sigma_{\mathrm{frozen}}^2(t) = \frac{\psi_0^2}{2\kappa_0}(1 - e^{-2\kappa_0 t})$. Assume that $\kappa(t)$, $\phi(t)$, and $\psi(t)$ are Lipschitz continuous in $t$ with constants $L_\kappa$, $L_\phi$, and $L_\psi$, i.e.,
$|\kappa(t)-\kappa(s)| \leq L_\kappa|t-s|$, $|\phi(t)-\phi(s)| \leq L_\phi|t-s|$, $|\psi(t)-\psi(s)| \leq L_\psi|t-s|$. Define $C_\mu = |a_0-\phi_0|$ and $C_\sigma = \psi_0^2/(2\kappa_0)$ then for any $t\in[0,T]$:

\textit{(i)} Mean error: $\varepsilon_\mu(t) = \mu_{\mathrm{varying}}(t) - \mu_{\mathrm{frozen}}(t)$ satisfies
\begin{equation}
|\varepsilon_\mu(t)| \;\leq\; \frac{\kappa_0 L_\phi + L_\kappa C_\mu}{L_\kappa}\!\left(e^{L_\kappa t^2/2} - 1\right)
\label{eq:mean_frozen_expression}
\end{equation}
\textit{(ii)} Variance error: $\varepsilon_{\sigma^2}(t) = \sigma^2_{\mathrm{varying}}(t) - \sigma^2_{\mathrm{frozen}}(t)$ satisfies
\begin{equation}
|\varepsilon_{\sigma^2}(t)| \;\leq\; \frac{L_\kappa\psi_0^2/(2\kappa_0) + L_\psi\psi_0}{L_\kappa}\!\left(e^{L_\kappa t^2} - 1\right)
\label{eqn:frozen_expression}
\end{equation}
For small $L_\kappa t^2 \ll 1$, both bounds simplify to $\mathcal{O}(t^2)$. 
\end{theorem}

\begin{proof}
The proof is divided into two parts: (i) mean error; and (ii) variance error.\\
\textit{(i) Mean error: } The frozen mean satisfies $\mu_{\mathrm{frozen}}(t) = \phi_0 + (a_0 - \phi_0)e^{-\kappa_0 t}$, and the varying mean satisfies the ODE $d\mu_{\mathrm{varying}}/{dt} = \kappa(t)\bigl(\phi(t) - \mu_{\mathrm{varying}}(t)\bigr)$, both with initial condition $a_0$. Let $\varepsilon_\mu = \mu_{\mathrm{varying}} - \mu_{\mathrm{frozen}}$, so $\varepsilon_\mu(0) = 0$. Subtracting the two ODEs and writing $\delta\kappa(t) = \kappa(t)-\kappa_0$, $\delta\phi(t)=\phi(t)-\phi_0$:
\begin{equation}
\begin{split}
\frac{d\varepsilon_\mu}{dt}
= -\kappa_0\varepsilon_\mu + \kappa_0\delta\phi(t) + \delta\kappa(t)\big(\phi(t) \\ 
\qquad -\mu_{\mathrm{frozen}}(t)\big)
- \delta\kappa(t)\varepsilon_\mu(t)
\end{split}
\end{equation}
Applying the integrating factor $e^{\kappa_0 t}$ and using $e^{-\kappa_0(t-s)}\leq 1$:
\begin{equation}
\begin{aligned}
|\varepsilon_\mu(t)|
&\leq \int_0^t 
\kappa_0|\delta\phi(s)|
+ |\delta\kappa(s)|\,|\phi(s)-\mu_{\mathrm{frozen}}(s)| \\
& \quad \qquad + |\delta\kappa(s)|\,|\varepsilon_\mu(s)| ds.
\end{aligned}
\end{equation}
By Lipschitz: $|\delta\kappa(s)|\leq L_\kappa s$, $|\delta\phi(s)|\leq L_\phi s$. Since $|\phi(s)-\mu_{\mathrm{frozen}}(s)| \leq C_\mu + L_\phi s$:
\begin{equation}
|\varepsilon_\mu(t)| \;\leq\; \frac{\kappa_0 L_\phi + L_\kappa C_\mu}{2}\,t^2 \;+\; L_\kappa \int_0^t s\,|\varepsilon_\mu(s)|\,ds
\end{equation}
Differentiating w.r.t. $t$ and applying the integrating factor
$e^{-L_\kappa t^2/2}$:
\begin{equation}
\frac{d}{dt}\!\left[e^{-L_\kappa t^2/2}\,|\varepsilon_\mu(t)|\right] \;\leq\; (\kappa_0 L_\phi + L_\kappa C_\mu)\,t\,e^{-L_\kappa t^2/2}
\end{equation}
Integrating from $0$ to $t$ and multiplying through by $e^{L_\kappa t^2/2}$ give Eqn.~\ref{eq:mean_frozen_expression}:
\begin{equation}
|\varepsilon_\mu(t)| \;\leq\; \frac{\kappa_0 L_\phi + L_\kappa C_\mu}{L_\kappa}\!\left(e^{L_\kappa t^2/2} - 1\right)
\end{equation}
Expanding for $L_\kappa t^2/2 \ll 1$ gives the leading
term $\tfrac{1}{2}(\kappa_0 L_\phi + L_\kappa C_\mu)\,t^2 = \mathcal{O}(t^2)$.\\
\textit{(ii) Variance error:} By Itô's formula, the variance of the OU process with time-varying coefficients satisfies the Riccati-type ODE~\cite{riccati1724animadversiones}
\begin{equation}
\frac{d\sigma^2}{dt} \;=\; -2\kappa(t)\,\sigma^2 \;+\; \psi(t)^2,    
\end{equation}
with frozen and varying versions both initialised at $\sigma^2(0)=0$. Let $\varepsilon_{\sigma^2} = \sigma^2_{\mathrm{varying}} - \sigma^2_{\mathrm{frozen}}$, so $\varepsilon_{\sigma^2}(0)=0$. Subtracting:
\begin{equation}
\frac{d\varepsilon_{\sigma^2}}{dt} \;+\; 2\kappa_0\,\varepsilon_{\sigma^2} \;=\; -2\delta\kappa(t)\,\sigma^2_{\mathrm{varying}}(t) \;+\; \delta(\psi^2)(t), 
\end{equation}
where $\delta(\psi^2)(t) = \psi(t)^2 - \psi_0^2$.
Applying the integrating factor $e^{2\kappa_0 t}$ and using
$e^{-2\kappa_0(t-s)}\leq 1$,
$\sigma^2_{\mathrm{varying}}(s) \leq C_\sigma + |\varepsilon_{\sigma^2}(s)|$,
$|\delta\kappa(s)|\leq L_\kappa s$, and $|\delta(\psi^2)(s)|\leq 2L_\psi\psi_0 s$:
\begin{equation}
|\varepsilon_{\sigma^2}(t)| \;\leq\; \left(L_\kappa C_\sigma + L_\psi\psi_0\right)t^2 \;+\; 2L_\kappa\int_0^t s\,|\varepsilon_{\sigma^2}(s)|\,ds.
\end{equation}
Differentiating w.r.t $t$ and applying the integrating factor
$e^{-L_\kappa t^2}$:
\begin{equation}
\frac{d}{dt}\!\left[e^{-L_\kappa t^2}|\varepsilon_{\sigma^2}|\right] \;\leq\; 2(L_\kappa C_\sigma+L_\psi\psi_0)\,t\,e^{-L_\kappa t^2}.
\end{equation}
Integrating from $0$ to $t$ and multiplying through by $e^{L_\kappa t^2}$:
\begin{equation}
|\varepsilon_{\sigma^2}(t)| \;\leq\; \frac{L_\kappa C_\sigma + L_\psi\psi_0}{L_\kappa}\!\left(e^{L_\kappa t^2}-1\right). 
\end{equation}
Substituting $C_\sigma = \psi_0^2/(2\kappa_0)$ gives Eqn.~\ref{eqn:frozen_expression}.
Expanding for $L_\kappa t^2\ll 1$ gives the leading term
$\bigl(L_\kappa\psi_0^2/(2\kappa_0) + L_\psi\psi_0\bigr)t^2 = \mathcal{O}(t^2)$.
\end{proof}

\begin{figure}[ht!]
\centering
\includegraphics[width=0.45\textwidth]{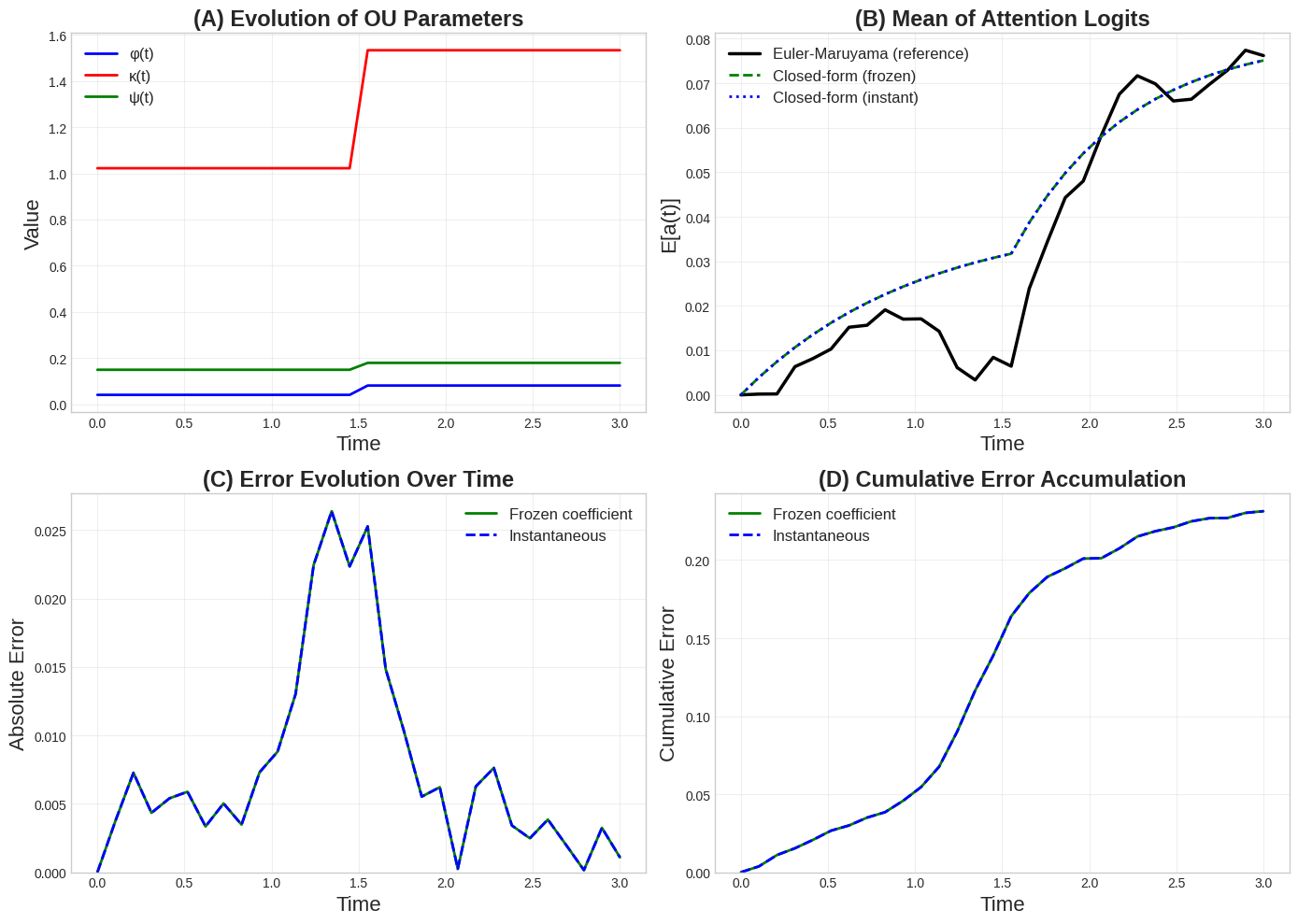}
\caption{Visualization of step-change evolution of NSAC.}\vspace{-10mm}
\label{fig:frozen}
\end{figure}

\begin{table}[h]
\centering
\caption{Frozen-coefficient approximation error analysis}
\label{tab:frozen-analysis}
\begin{tabular}{l c c c c c}
\hline
\makecell{\textbf{Variation}} & 
\makecell{\textbf{Mean} \\ (\%)} & 
\makecell{\textbf{Max} \\ (\%)} & 
\makecell{\textbf{$\mathbf{\phi}$} \\ (\%)} & 
\makecell{\textbf{$\mathbf{\kappa}$} \\ (\%)} & 
\makecell{\textbf{$\mathbf{\psi}$} \\ (\%)} \\
\hline
Slow        & 2.32 & 4.63  & 0.50 & 0.26 & 0.40  \\
Fast        & 4.81 & 5.16  & 1.63 & 0.88 & 1.10  \\
Step        & 5.36 & 5.671  & 3.51 & 2.11 & 0.96  \\
Oscillatory & 5.87 & 6.57  & 4.28 & 3.10 & 4.54  \\
\hline
\end{tabular}
\end{table}
\textbf{Empirical Analysis:} We now empirically analyze the effect of frozen coefficients. Four different data regimes are considered: (i) slow variation; (ii) fast variation; (iii) step changes; and (iv) oscillatory behavior. The Figure~\ref{fig:frozen} shows the qualitative illustration of step changes and Table~\ref{tab:frozen-analysis} summarizes the quantitative results. For typical slowly varying sequences, where $\phi_i$, $\kappa$, and $\psi$ vary by 0.50\%, 0.26\%, and 0.40\%, respectively, the frozen approximation achieves a mean error of only 2.32\% and a maximum error of 4.63\%. Under more challenging conditions, the mean error increases to 4.81\% for fast variation, 5.36\% for step changes, and for oscillatory regimes remain well-behaved with a mean error of 5.87\%.\\
\emph{Interpretation:} The values of $\phi$, $\kappa$, and $\psi$ trained using $\mathcal{NN}_{\text{backbone}}$ serve as finite-difference estimates of $L_\phi$, $L_\kappa$, and $L_\psi$, respectively. The predicted $\mathcal{O}(t^2)$ error bounds are validated empirically across all four regimes. In the slow-variation regimes, the error is approximately $2.32\%$, consistent with the expected quadratic behavior. By contrast, the fast-variation and step-change regimes exhibit larger errors ($4.81\%$--$5.36\%$), reflecting a larger effective Lipschitz constant and, consequently, larger analytical prefactors. Overall, the frozen-coefficient approximation remains accurate across a diverse range of regimes, with the error remaining bounded even under substantial parameter variations. Since real-world irregularities rarely exceed $10\%$, these results support the use of the frozen-coefficient closed-form solution for NSAC.

\subsection{Proof for Theorem~\ref{theorem:state_stability}}\label{appendix:state_stability}

The proof uses a standard forward-invariance argument: we verify that
the vector field in drift ODE ($d\mu_t^{(i)}/{dt} = \sum_{j=1}^{M} \kappa_{i,j}(\Phi_{i,j} - \mu_t^{(i)})$) is non-positive at the upper
boundary and non-negative at the lower boundary of
$[\Phi_{\min},\Phi_{\max}]$.

\textit{Upper bound:} At $\mu = \Phi_{\max}$:
\begin{equation}
  \frac{d\mu_t^{(i)}}{dt}\Bigg|_{\mu\,=\,\Phi_{\max}} = \sum_{j=1}^{M}\kappa_{i,j} \left(\Phi_{i,j} - \Phi_{\max}\right) \leq 0,
  \label{eq:upper-boundary}
\end{equation}
since $\Phi_{i,j} \leq \Phi_{\max}$ for every $j$ by definition,
and each weight $f_{\kappa}([\mathbf{q_i};\,\mathbf{k_j}]) > 0$.
Hence the drift is non-positive at the upper boundary, and
$\mu_t^{(i)}$ cannot exceed $\Phi_{\max}$.

\textit{Lower bound:} At $\mu = \Phi_{\min}$:
\begin{equation}
  \frac{d\mu_t^{(i)}}{dt}\Bigg|_{\mu\,=\,\Phi_{\min}} = \sum_{j=1}^{M} \kappa_{i,j}(\Phi_{i,j} - \Phi_{\min}) \;\geq\; 0,
  \label{eq:lower-boundary}
\end{equation}
since $\Phi_{i,j} \geq \Phi_{\min}$ for every $j$.
Hence the drift is non-negative at the lower boundary, and
$\mu_t^{(i)}$ cannot fall below $\Phi_{\min}$.

\textit{Special case $M=1$:} The drift equation collapse to $d\mu_t^{(i)}/dt = \kappa(\phi - \mu_t^{(i)})$ and applying the same boundary check:
\begin{align}
  \frac{d\mu_t^{(i)}}{dt}\Bigg|_{\mu\,=\,\phi_{\max}}
  &= \kappa(\phi - \phi_{\max}) \leq 0,
  \label{eq:M1-upper}\\[4pt]
  \frac{d\mu_t^{(i)}}{dt}\Bigg|_{\mu\,=\,\phi_{\min}}
  &= \kappa(\phi - \phi_{\min}) \geq 0,
  \label{eq:M1-lower}
\end{align}
since $\phi \leq \phi_{\max}$ and $\phi \geq \phi_{\min}$ by
definition of the range of $f_{\phi}$. With the $\tanh$ nonlinearity enforcing $\phi \in (-1,1)$, we have $\phi_{\min} \leq 0 \leq \phi_{\max}$, so the implementation's initialization $\mu_0^{(i)} = 0$ satisfies $\mu_0^{(i)} \in [\phi_{\min},\phi_{\max}]$.
The trajectory converges monotonically to $\phi$ and remains within $[\phi_{\min},\phi_{\max}]$ throughout.
\begin{remark}
This proof establishes the boundedness of the mean attention-logit trajectory, ensuring that, under a positive $\kappa$, the dynamics remain well-behaved and bounded.
\end{remark}

\section{Evaluation Details}
In this section, we provide comprehensive definitions of evaluation metrics used to assess the model.

\subsection{Evaluation Metrics}\label{appendix:metrics}
\subsubsection{Mean Squared Error (MSE)}
MSE is a widely used metric for evaluating regression models, measuring the average squared difference between predicted values $\hat{y}_i$ and true targets $y_i$:
\begin{equation}
\text{MSE} = \frac{1}{N} \sum_{i=1}^{N} (\hat{y}_i - y_i)^2,
\end{equation}
where $N$ is the total number of samples. Lower MSE values indicate that predictions are closer to the true targets.

\subsubsection{Negative-Log Likelihood (NLL)} 
NLL is a strictly proper scoring rule that evaluates the quality of probabilistic regression by measuring how much probability mass the predicted Gaussian distribution assigns to the true target. For a predicted mean $\mu_i$, log-standard deviation $s_i = \log \sigma_i$, and true target $y_i$, the NLL is defined as
\begin{equation}
\text{NLL} =  \frac{1}{N}\sum_{i=1}^{N} \left[ \frac{1}{2}\log(2\pi) +s_i + \frac{(y_i-\mu_i)^2}{2e^{2s_i}} \right] 
\end{equation}
Where N is the total number of samples. Intuitively, NLL jointly evaluates the predictive accuracy and uncertainty estimates by penalizing both inaccurate prediction and poorly estimated predictive variance. Because continuous probability densities are not bounded by one, NLL values may be negative when the predictive variance is small and predictions closely match the targets. Therefore, lower NLL values indicate that the predicted distribution assign higher likelihood to the observed target, reflecting better probabilistic modeling.

\subsubsection{Continuous Ranked Probability Score (CRPS)}
CRPSis a strictly proper scoring rule that assesses both calibration and sharpness. 
For a predicted cumulative distribution function (CDF) $\Phi_i$ and true target $y_i$, the CRPS is defined as
\begin{equation}
\text{CRPS}(\Phi_i, y_i) = \int_{-\infty}^{\infty} \big(\Phi_i(t) - \mathbf{1}(t \ge y_i)\big)^2 \, dt,
\end{equation}
where $\mathbf{1}(\cdot)$ is the indicator function. Intuitively, CRPS measures the squared distance between the predicted CDF and the empirical step function at the observation. Lower CRPS values indicate more accurate and well-calibrated probabilistic forecasts, as the metric jointly rewards correct central tendency and uncertainty estimation.

\subsubsection{Expected Calibration Error (ECE)} 
ECE is commonly used as probabilistic classification metric but we adopt it for continuous regression prediction. For each prediction $(\mu_i, \sigma_i)$ and true target $y_i$, we compute the standardized absolute error $z_i = |y_i - \mu_i|/\sigma_i$ and map it to a confidence score $c_i = 2\Phi(z_i) - 1$, where $\Phi$ is the standard normal cumulative distribution function (CDF). Confidence scores are partitioned into $B$ uniform bins (10 bins in our implementation), and for each bin we calculate the empirical coverage, defined as the fraction of samples with $z_i$ below the bin midpoint. The ECE is then the weighted average absolute difference between predicted confidence and empirical coverage across bins.
\begin{equation}
\mathrm{ECE} = \frac{1}{N} \sum_{k=1}^{B} \left| \hat{p}_k - \hat{c}_k \right| \cdot |S_k|  
\end{equation}
where, $N$ be the total number of samples, and $S_k$ the set of samples in bin $k$. Lower ECE indicates better alignment between predicted uncertainties and observed errors, enabling reliable uncertainty intervals. 

\begin{table}[t]
\centering
\caption{Summary of Key Hyperparameters of All Experiments}
\label{tab:hparams}
\resizebox{1.0\columnwidth}{!}{%
\begin{tabular}{lccc}
\toprule
\textbf{Param.} & \textbf{Exp.~1 to 4} & \textbf{Udacity} & \textbf{CarRacing} \\
\midrule
Conv-layers & -- & 5×\textbf{2D}(24--64@5--3, ELU) & 3×TD-\textbf{2D}(32--640@3--5)   \\
NSAC/NAC & 64-d, 16h & 100-d, 20h & 64-d, 16h \\
Sparsity & 0.5 & 0.5 & 0.5 \\
OOD params.~($\mu_{pert}$,$\sigma_{pert}$) & (0,5.0) & (0,5.0) & (0,5.0) \\
Dense & -- & -- & -- \\
MC-steps &5 & 5 & 5 \\
Opt. & AdamW & AdamW & AdamW \\
LR & 0.001 & 0.0001 & 0.0001 \\
Batch & 64 & 40 & 32 \\
Epochs & 100 & 10 & 20 \\
\bottomrule
\end{tabular}}
\begin{minipage}{\columnwidth}
\footnotesize
\textbf{Note:} TD = TimeDistributed; \textbf{1D}/\textbf{2D} = Conv1D/2D; $d$ = model dimension; $h$ = attention heads. \\
\textbf{Baseline clarification:} For Deep-Ensemble we use $\times$3 models and for MC-Dropout, we placed dropout layer as penultimate layer with drop\_rate=0.1 with MC steps same as NSAC. 
\end{minipage}
\vspace{-3mm}
\end{table}

\subsection{Experiment Details}\label{appendix:experimental_detail}
In this section, we briefly overview the methodology of each experiments including dataset details, preprocessing and neural network architecture utilized to carry the experiments. The Table~\ref{tab:hparams} presents the neural networks architecture/hyperparameters utilized during the training/testing.

\subsubsection{Multivariate Regression}
\textbf{Dataset Explanation:} The Boston Housing dataset comprises 506 samples of residential areas in the Boston suburbs, with each sample described by 13 features, including crime rate, average number of rooms per dwelling, and accessibility to highways. The target variable is the median value of owner-occupied homes. The Kin8nm dataset contains 8,192 samples, where each instance consists of 8 continuous input features derived from physical measurements and a single continuous target variable. It is a synthetic benchmark dataset designed to evaluate regression models on smooth but nonlinear input-output relationships. \\
\textbf{Preprocess:} Prior to training both dataset were normalized using \texttt{MinMaxScaler} function. \\
\textbf{Neural Network:} A single NSAC layer with a model dimension of 64, 16 heads, Top-\emph{K} of 8, and sparsity of 0.5 was used for both tests, with all baselines configured with the same specifications.

\subsubsection{Multivariate Long-range forecasting}
\textbf{Dataset Explanation:} We utilized two well-known long-range modeling datasets: (i) ETTm1; and (ii) Jena-Climate. ETTm1 consists of 7 features and is used for long-term oil temperature forecasting. It provides data at 15-minute intervals, capturing real-world temporal patterns for time-series analysis. Jena-Climate contains 14 features capturing various atmospheric and environmental measurements. It provides data at 10-minute intervals for modeling and forecasting climate-related time series. \\
\textbf{Preprocess:} Prior to training both dataset were normalized using \texttt{MinMaxScaler} function. \\
\textbf{Neural Network:} A single NSAC layer with a model dimension of 64, 16 heads, a top-k of 8, and a sparsity of 0.5 was used for both tests, with all baselines configured with the same specifications.

\subsubsection{Industry~4.0}
\textbf{Dataset Explanation:} We utilized three well-known bearing degradation datasets to conduct our experiments: XJTU-SY, PRONOSTIA, and HUST. Each dataset contains 1D vibrational signals captured from at least two sensors under three distinct operating conditions. To train the neural network, we used the first three bearings from operating condition 1 of XJTU-SY and evaluated the model in zero-shot manner on PRONOSTIA and HUST, also using first operating condition. This setup allows us to assess OOD calibration and zero-shot testing performance of NSAC.\\
\textbf{Preprocessing:} The raw vibrational signals needs to be converted into a meaningful set of features to capture degradation patterns. We followed the feature engineering approach proposed in \cite{razzaq2025carle}, which employs a wavelet transform to extract seven physically interpretable features, including energy, dominant frequency, skewness, and others. \\
\textbf{Neural Network:} A single NSAC layer with a model dimension of 64, 16 heads, a top-k of 8, and a sparsity of 0.5 was used for both tests, with all baselines configured with the same specifications.

\subsubsection{Autonomous Vehicle}
\textbf{Data Curation:} We collected training data for both the Udacity and OpenAI-CarRacing environments by manually playing each game for 30 minutes, recording images and the corresponding steering values. \\
\textbf{Preprocessing:} For Udacity, we applied the preprocessing steps proposed in~\cite{shibuya_car_behavioral_cloning}, while for OpenAI-CarRacing, we did not apply any preprocessing.\\
\textbf{Neural Network:} Both neural networks are trained end-to-end and consist of a series of convolutional heads that extract spatial features, which are then passed to NSAC or baseline models for temporal modeling. Detailed network architecture and all hyperparameters are provided in Table~\ref{tab:hparams}.

\begin{figure}[t]
\centering
\includegraphics[width=1.0\columnwidth]{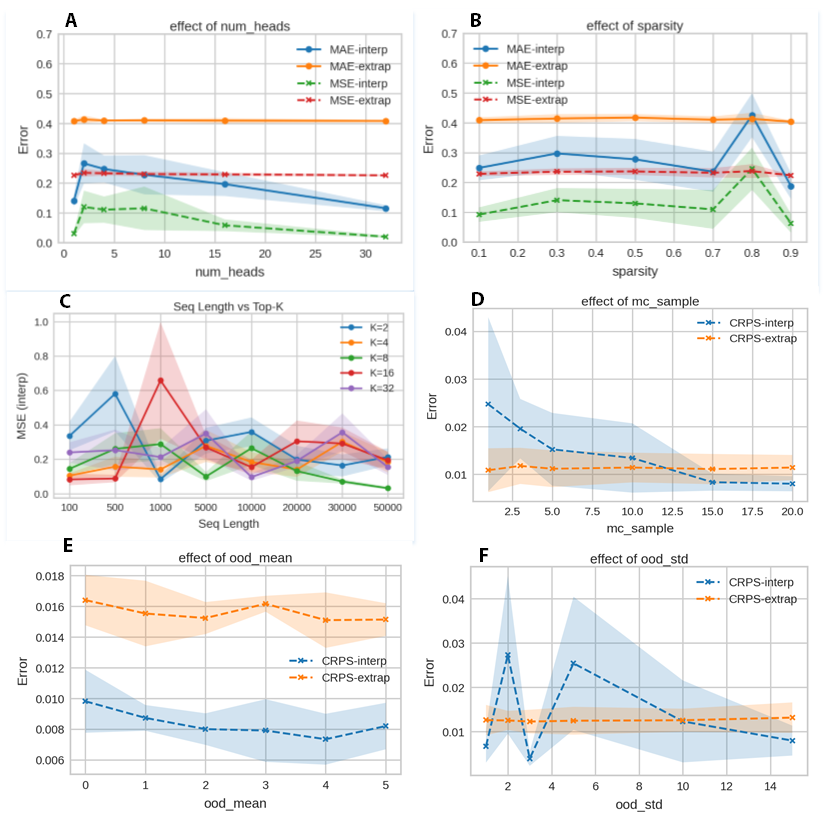}
\caption{Visualization of impact of key hyperparameters of NSAC. (\textbf{A}) No. of attention heads; (\textbf{B}) Sparisity level; (\textbf{C}) Sequence length vs. Top-\emph{K} selection; (\textbf{D}) No. of MC-Samples; (\textbf{E}) OOD\_mean ($\mu_{\mathrm{pert}}$); and (\textbf{F}) OOD\_std ($\sigma_{\mathrm{pert}}$).} \vspace{-5mm}
\label{fig:key_hyp}
\end{figure}

\section{Additional Experiments}\label{appendix:additional_experiments}
In this section, we provide additional experiments including: (i) Ablation study of key-hyperparameters; (ii) Closed-loop trajectory analysis of AVs; and (iii) Scalability \& Efficiency analysis.
\subsection{Ablation Analysis on Key Parameters}\label{appendix:ablation_analysis}
In this experiment, we analyze the effects of key NSAC hyperparameters: (i) the number of attention heads (1--32); (ii) sparsity (0.1--0.9); (iii) Top-\emph{K} selection (2--32) across sequence lengths (100--50000); (iv) MC samples (1--20); effects of (v) $\mu_{\mathrm{pert}}$ (0--5); and (vi) $\sigma_{\mathrm{pert}}$ (0--15). We conduct this evaluation on a modified irregular spiral dataset while fixing the $d_{model}$ to 64. Figure~\ref{fig:key_hyp} presents the quantitative results of this study.\\
For the first two analyses (Figure~\ref{fig:key_hyp}(A,B)), we report both MSE and MAE under the interpolation and extrapolation regimes. Both the number of heads and sparsity exhibit non-monotonic behavior. Increasing the number of heads improves performance, with 32 heads achieving the lowest MSE and MAE across both regimes. Similarly, although 90\% sparsity achieves the lowest MSE, excessive sparsity may reduce representational capacity. Sparsity levels between 0.2 and 0.7 yield nearly identical performance, with 0.5 providing the most balanced trade-off. Figure~\ref{fig:key_hyp}(C) shows the effect of Top-\emph{K} selection across varying sequence lengths under the interpolation regime. Among all the settings, \emph{K}=8 consistently achieves the lowest error while maintaining a moderate computational cost. For the last three experiments, we report the CRPS. Increasing the number of MC samples improves performance, as shown in Figure~\ref{fig:key_hyp}(D), although the gains diminish and performance saturates beyond 15 samples. Both $\mu_{\mathrm{pert}}$ (Figure~\ref{fig:key_hyp}(E)) and $\sigma_{\mathrm{pert}}$ (Figure~\ref{fig:key_hyp}(F)) exhibit non-monotonic effects, with the best performance achieved at $\mu_{\mathrm{pert}} = 4$ and $\sigma_{\mathrm{pert}} = 5$. However, these settings require careful tuning. Therefore, we choose $\mu_{\mathrm{pert}} = 0$ and $\sigma_{\mathrm{pert}} = 5$ as the default configuration and use them for all experiments.

\begin{figure*}[t]
\centering
\includegraphics[width=0.85\textwidth]{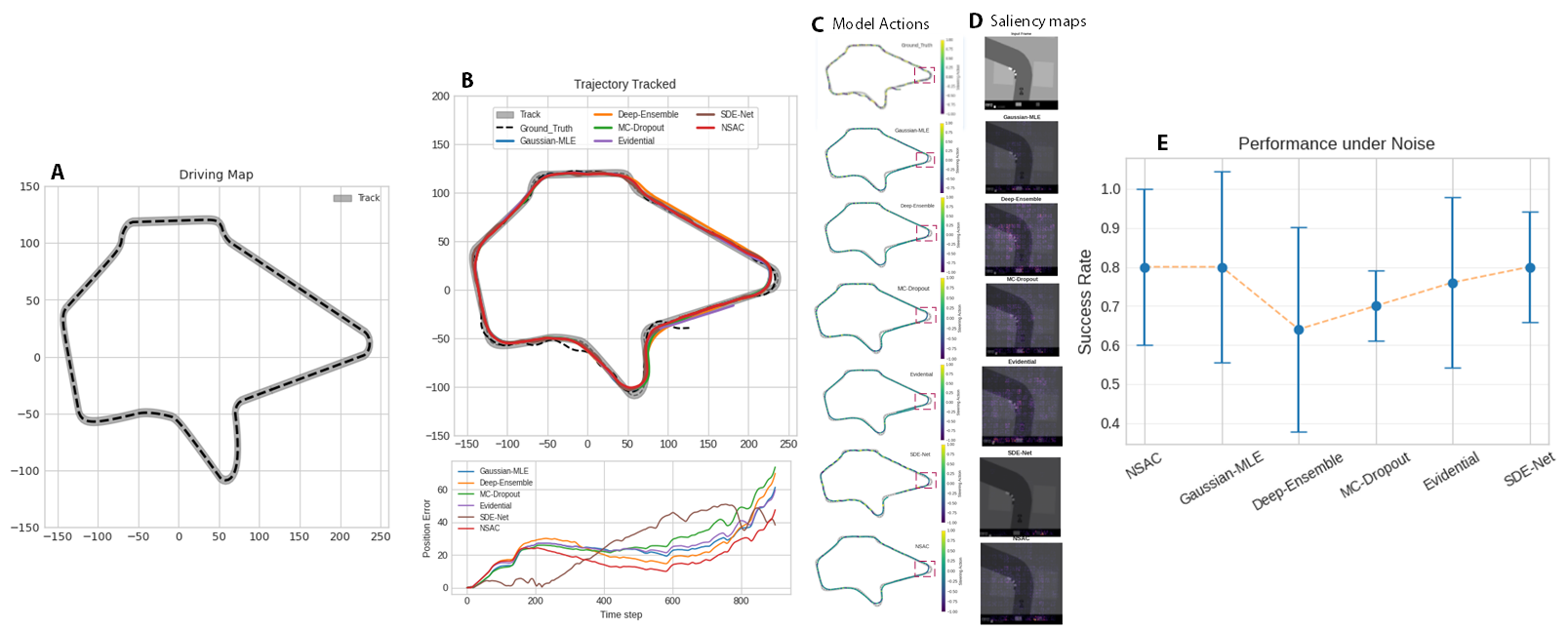}
\caption{Closed-loop driving analysis on OpenAI-CarRacing. (\textbf{A}) Test drive map; (\textbf{B}) Trajectory followed by positional error to the center of the road; (\textbf{C}) Individual model actions at each step on the map; (\textbf{D}) Visual saliency maps in the highlighted region; (\textbf{E}) Noise test.}\vspace{-6mm}
\label{fig:closed-loop}
\end{figure*}

\subsection{Closed-loop Analysis of AVs}\label{appendix:closed-loop}
Beyond the predictive metrics, we further conduct a closed-loop evaluation on the OpenAI CarRacing environment using a controlled test lap configuration, as shown in Figure~\ref{fig:closed-loop}(A). This setting enables a qualitative analysis of steering behavior under sequential decision-making. The figure illustrates the trajectories generated by each model, along with the corresponding position error relative to the road centerline. The NSAC achieves the lowest positional error, indicating that it most consistently remains aligned with the center of the road.\\
Figure~\ref{fig:closed-loop}(B) presents the sequence of actions produced by each model. Notably, the ground-truth trajectory itself exhibits distortions, which propagate into the learned baselines. In contrast, the NSAC, implemented with the penultimate NSAC layer, exhibits a self-regulating behavior that filters out such distortions and maintains close adherence to the road structure. Although SDE-Net is able to complete the lap, its predicted actions are severely distorted, indicating instability in sequential control.\\
In Figure~\ref{fig:closed-loop}(C), we further analyze saliency maps using Grad-CAM~\cite{selvaraju2017grad} to identify the regions of the input that influence action decisions. The NSAC consistently focuses on the full road curvature when actions are selected, suggesting that causally meaningful attention is aligned with the underlying driving structure. In contrast, other baselines exhibit less consistent and often scattered attention patterns across irrelevant regions of the environment. SDE-Net in particular appears to exhibit weak and unreliable visual grounding during decision-making, which explains its degraded action quality.\\
We additionally evaluate robustness to environmental perturbations by modifying the color scheme of the simulation environment while preserving the underlying road trajectory. This introduces visual noise without altering the control task. We conduct 10 runs with 10 trials each and report the mean success rate and standard deviation. All the models perform reasonably well under this setting. NSAC achieves an average success rate of approximately 80\%, which is comparable to that of GMLE and SDE-Net, while all the remaining baselines exceed a 60\% success rate.

\begin{table}[ht!]
\centering
\caption{Run-Time \& Memory Results}
\label{tab:run_time}
\resizebox{1.0\columnwidth}{!}{%
\begin{tabular}{lccccc}
\toprule
\textbf{Model} & \makecell{\textbf{Run-Time}\\(s)} & \makecell{\textbf{Throughput}\\(seq/s)} & \makecell{\textbf{Peak Memory} \\(MB)} & \makecell{\textbf{Params.}} \\
\midrule
GMLE & 14.47\textsuperscript{\scriptsize $\pm$0.21}  & 0.07 & 75.78 & 36218 \\
DE ($N_m=3$) & 39.02\textsuperscript{\scriptsize $\pm$0.18}  & 0.07 & 75.98 & 36218 \\
MCD ($N_{mc}=5$) & 14.09\textsuperscript{\scriptsize $\pm$0.23}  & 0.07 & 75.89 & 36218 \\
DER & 14.21\textsuperscript{\scriptsize $\pm$0.23}  & 0.07 & 76.00 & 36416 \\
\midrule
NSAC ($N_{mc}=5$)& 14.14\textsuperscript{\scriptsize $\pm$0.46}  & 0.07 & 75.94 & 37278 \\
\bottomrule
\end{tabular}}
\begin{minipage}{\columnwidth}
\footnotesize
\vspace{1mm}
\textbf{Note:} $N_m$ (No. of models);  $N_{mc}$ (No. of MC samples).\vspace{-5mm}  
\end{minipage}
\end{table}

\subsection{Scalability \& Efficiency Analysis}
We evaluate the efficiency and scalability of NSAC’s solver-free forward-pass update on configurations used for spiral with synthetic inputs (batch size 1, sequence length of 1000). Experiments are conducted on Google Colab T4-GPU and compared against all post-hoc methods using solver-free NAC as the penultimate layer for fair comparison. We report the mean and standard deviation over five forward passes for runtime, throughput, and parameter count (see Table~\ref{tab:run_time}). NSAC achieves a mean run-time comparable to the strongest baseline (14.13 s) with a throughput of 0.07 and 75.94 MB peak memory. It incurs slightly higher parameter count than baselines (37288 vs. 36218) due to an additional backbone head.


\end{document}